\documentclass[sigconf]{acmart}
\AtBeginDocument{%
  }

\usepackage{booktabs}
\usepackage{colortbl}
\usepackage{dsfont}
\usepackage{amsmath}
\usepackage{subcaption}
\usepackage{multirow}
\usepackage{graphicx}
\usepackage{thm-restate}
\makeatletter
\renewcommand\thmt@autorefsetup{%
  \@xa\def\csname\thmt@envname autorefname\@xa\endcsname\@xa{\thmt@thmname}%
}
\makeatother
\usepackage{enumitem}
\usepackage{algorithm}
\usepackage{algpseudocode}

\newtheorem{definition}{Definition}

\newtheorem{assumption}{Assumption}
\newtheorem{theorem}{Theorem}
\newtheorem{lemma}{Lemma}
\newtheorem{proposition}{Proposition}

\newcommand{\name}{\textbf{GSCM-PFN}}
\newcommand{\SCM}{Joint Feature-Edge-Label SCM}
\newcommand{\scm}{joint feature-edge-label SCM}
\newcommand{\SCMDAG}{Prior-Guided Node-Edge Dependency Modeling}
\newcommand{\scmdag}{prior-guided node-edge dependency modeling}
\newcommand{\Architecture}{Feature-Permutation-Invariant Dual-Path Graph Network}
\newcommand{\architecture}{feature-permutation-invariant dual-path graph network}
\newcommand{\scmm}{graph causal-relation decoupling mechanism}
\newcommand{\gt}{feature-permutation-equivariant graph transformer}

\newcommand{\best}[1]{\textbf{#1}} %
\newcommand{\secbest}[1]{\underline{#1}} %

\setcopyright{acmlicensed}
\copyrightyear{2018}
\acmYear{2018}
\acmDOI{XXXXXXX.XXXXXXX}
\acmConference[Conference acronym 'XX]{Make sure to enter the correct
  conference title from your rights confirmation email}{June 03--05,
  2018}{Woodstock, NY}
\acmISBN{978-1-4503-XXXX-X/2018/06}

\acmSubmissionID{5119}

\begin{document}

\title{Learning a Prior-Data Fitted Network for \\
Graph Foundation Models via Structural Causal Model Priors}

\author{Ziwei Zhang}
\affiliation{
  \institution{Beihang University}
  \city{Haidian District}
  \state{Beijing}
  \country{China}
}
\email{zwzhang@buaa.edu.cn}

\author{Tianyin Liao}
\affiliation{
  \institution{Nankai University}
  \city{Binhai New Area}
  \state{Tianjin}
  \country{China}
}
\email{1120230329@mail.nankai.edu.cn}

\author{Chunyu Hu}
\affiliation{
  \institution{Nankai University}
  \city{Binhai New Area}
  \state{Tianjin}
  \country{China}
}
\email{huchunyu@mail.nankai.edu.cn}

\author{Ge Lan}
\affiliation{
  \institution{Nankai University}
  \city{Binhai New Area}
  \state{Tianjin}
  \country{China}
}
\email{grettelan@mail.nankai.edu.cn}

\author{Nuo Xu}
\affiliation{
  \institution{Tsinghua University}
  \city{Haidian District}
  \state{Beijing}
  \country{China}
}
\email{xunuo0109@hotmail.com}

\author{Yi Huang}
\affiliation{
  \institution{Beihang University}
  \city{Haidian District}
  \state{Beijing}
  \country{China}
}
\email{yihuang@buaa.edu.cn}

\author{Qingyun Sun}
\affiliation{
  \institution{Beihang University}
  \city{Haidian District}
  \state{Beijing}
  \country{China}
}
\email{sunqy@buaa.edu.cn}

\author{Xingxuan Zhang}
\affiliation{
  \institution{Tsinghua University}
  \city{Haidian District}
  \state{Beijing}
  \country{China}
}
\email{xingxuanzhang@hotmail.com}

\author{Peng Cui}
\affiliation{
  \institution{Tsinghua University}
  \city{Haidian District}
  \state{Beijing}
  \country{China}
}
\email{cuip@tsinghua.edu.cn}

\author{Jianxin Li}
\affiliation{
  \institution{Beihang University}
  \city{Haidian District}
  \state{Beijing}
  \country{China}
}
\email{lijx@buaa.edu.cn}

\renewcommand{\shortauthors}{Trovato et al.}

\begin{abstract}
  Graph Foundation Models (GFMs) have recently emerged as a promising paradigm for learning transferable representations across different graphs and domains. However, existing GFMs face notable limitations, including limited pre-training scale or reliance on aligning with textual information in Large Language Models (LLMs). To overcome these limitations, recent studies have drawn inspiration from prior-data fitted networks (PFNs), which learn from large synthetic datasets and support multi-domain prediction backed by causal learning. However, existing graph PFNs encounter severe difficulties in modeling causal relationships within and across node and edge sets, generating diverse graph structures, and designing suitable architectures that can learn from complex synthetic datasets. To address these challenges, we propose \name~(\underline{G}raph \underline{SCM}-based \underline{P}rior-Data \underline{F}itted \underline{N}etwork), a graph PFN pre-trained on synthetic data generated from an SCM-based prior that jointly models nodes and graph topology. Specifically, we first design a \scm~to guide the generation of nodes and edges in synthetic graphs, enabling the PFN to capture causal relationships within and across node and edge sets. Next, we introduce a \scmdag~module that instantiates diverse synthetic graph generation strategies. We prove that our SCM-based graph generation strategy is more expressive than existing priors. Finally, we propose a \architecture~to learn from diverse synthetic graph data with theoretical guarantees. Experiments on 23 real-world benchmarks show that \name~outperforms both conventional GFMs and PFN-based GFMs, without utilizing real-world graph data for pre-training. Our code is available at \url{https://anonymous.4open.science/r/GSCM-PFN-A3A3/}.
\end{abstract}

\ccsdesc[500]{Mathematics of computing~Graph algorithms}
\ccsdesc[500]{Computing methodologies~Neural networks}

\keywords{Graph Foundation Model, Prior-Data Fitted Network, Graph Machine Learning}

\maketitle

\section{Introduction}\label{sec:intro}
Graph Foundation Models (GFMs)~\cite{liu2025graph,zhao2025graphany,finkelshtein2026equivariance} have recently emerged as a promising paradigm for learning transferable representations across different graphs and domains. Unlike graph neural networks (GNNs)~\cite{velickovic2018graph,shchur2018pitfalls}, which are often trained for specific tasks or domains, GFMs are pre-trained on diverse graph data and can be adapted to a broad range of downstream problems. This pre-training then adapting paradigm offers a unified and flexible approach to universal graph learning, with applications including molecular property prediction~\cite{zhao2023gimlet} and knowledge graph reasoning~\cite{galkin2024towards}. On the Web, GFMs can be applied to downstream tasks such as social network analysis~\cite{zhu2025graphclip}, recommender systems~\cite{gong2023unified,wang2025graph}, knowledge graph mining~\cite{10.1145/2488388.2488425}, etc.

However, existing GFMs face notable limitations. For example, GNN-based GFMs~\cite{zhao2025graphany,finkelshtein2026equivariance} depend heavily on large-scale pre-training data from diverse graph domains~\cite{liu2025graph}, yet such data may be insufficient to support generalization across domains and tasks. Meanwhile, LLM-based GFMs~\cite{liuone,chen2024exploring} often rely on aligning with textual information, which can lead to inconsistent data priors when applied to graphs with few textual attributes~\cite{eremeev2026turning,eremeev2026graphpfn,liao2026tfmlinker}.

To overcome these limitations, recent studies~\cite{eremeev2026turning,hu2026surprisingly,eremeev2026graphpfn} have drawn inspiration from the success of prior-data fitted networks (PFNs)~\cite{hollmanntabpfn, zhang2025limix}, which achieve multi-domain prediction across diverse tabular datasets. PFNs, such as TabPFN~\cite{hollmann2025accurate} and LimiX-2~\cite{zhang2026limix}, are pre-trained on diverse synthetic datasets to learn posterior predictive distributions induced by a structural causal model (SCM) prior, thereby enabling effective in-context learning, particularly for numerical attributes. One line of work~\cite{liao2026tfmlinker,gao2026kgpfn} has adapted PFNs pre-trained in the tabular domain to graph foundation models, as in G2T-FM~\cite{eremeev2026turning} and TAG~\cite{hayler2025bringing}. Although these models have achieved reasonable performance, they still rely heavily on hand-crafted features and lack large-scale pre-training on diverse graph data, limiting their ability to capture complex graph patterns~\cite{eremeev2026graphpfn}. 

Another line of work~\cite{eremeev2026graphpfn,choi2026nodepfn} develops graph-specific priors for multi-domain node classification that natively incorporate topology, as in GraphPFN~\cite{eremeev2026graphpfn} and NodePFN~\cite{choi2026nodepfn}. However, they face three unresolved major challenges. First, although both models use SCM priors to generate node features, they rely on the Erd\H{o}s--R\'enyi (ER) model and methods in the stochastic block model (SBM) family to generate graph topology~\cite{erdds1959random,holland1983stochastic,karrer2011stochastic,binkiewicz2017covariate}, rather than on a unified SCM that jointly models causal dependencies within node and edge sets and between them. Such an SCM could provide a more principled basis for generalization, since causal dependencies both exist within node and edge sets and between them~\cite{bramoulle2009identification,centola2010spread,steglich2010dynamic,leskovec2008microscopic}. For example, node attributes can affect other nodes and give rise to edges, while existing edges can alter node attributes and lead to further edges.

Second, generating diverse synthetic graph data is more challenging than generating tabular data, since graph topologies can exhibit complex dependencies among edges. The existing methods rely on ER-based and SBM-family priors, which do not explicitly capture several such dependencies, including triadic closure~\cite{wang2024neural}, degree assortativity~\cite{newman2002assortative}, and continuous core-periphery structure~\cite{borgatti2000models}. The synthetic graphs therefore cover a narrower range of topologies.

Lastly, designing an architecture that can learn from such complex synthetic datasets is also a challenge. These datasets combine heterogeneous feature semantics which differ across graphs with topological dependencies that vary widely in range. NodePFN does not guarantee feature-permutation invariance, so it fails to model inputs whose feature dimensions do not share a fixed meaning. GraphPFN augments the tabular LimiX model by stacking message-passing layers in series, but this design cannot capture complex and diverse graph structures.

To address these challenges, we propose \name~(\underline{G}raph \underline{SCM}-based \underline{P}rior-Data \underline{F}itted \underline{N}etwork), a PFN-based multi-domain graph foundation model. Specifically, \name~is pre-trained on synthetic data generated from an SCM-based prior that jointly models nodes and graph topology through three tailored modules to address the challenges outlined above. First, we design a \scm~to coordinate the generation of nodes and edges in synthetic graphs. Through a \scmm, the SCM lets the model capture causal relationships within the node set and the edge set, as well as between them. Next, we introduce a \scmdag~module that carries out the synthetic graph generation process specified by the SCM. This module yields diverse generation strategies by composing multiple layers, which exchange information between nodes and edges at the sample level. We prove that this strategy is strictly more expressive than the ER and SBM priors used in existing methods. Finally, we propose a \architecture~for learning on synthetic graphs sampled from complex SCM priors. The network combines a provably feature-permutation-invariant design with graph transformer layers, which makes it suitable for learning from diverse synthetic graph data with heterogeneous feature semantics. We evaluate our method on real-world benchmarks and the results show that \name~outperforms both conventional GFMs and PFN-based GFMs, without using real-world graph data for pre-training. 

Our contributions are summarized as follows:
\begin{itemize}[leftmargin=1.5em]
    \item We propose \name, a PFN-based multi-domain graph foundation model pre-trained on causally inspired synthetic graphs instead of real-world graph data.
    \item We develop a \scm~to guide the \scmdag~module for diverse graph generation, modeling causal dependencies within and between the node and edge sets. We prove that our strategy is more expressive than existing ER and Bernoulli SBM priors.
    \item We design a \architecture~that handles heterogeneous feature semantics and captures graph topology with theoretical guarantees. %
    \item \name~outperforms both conventional GFMs and PFN-based GFMs on 23 real-world benchmarks, while only utilizing synthetic data during pre-training.
\end{itemize}

\section{Preliminaries and Problem Formulation}\label{sec:problem}
Given a graph $\mathcal{G}=(\mathcal{V},\mathbf{A},\mathbf{X},\mathbf{y}_\mathrm{train})$, where $\mathcal{V}$ is the node set, $\mathbf{A}\in\{0,1\}^{|\mathcal{V}|\times|\mathcal{V}|}$ is the adjacency matrix, $\mathbf{X}\in\mathbb{R}^{|\mathcal{V}|\times d}$ is the node feature matrix with feature dimension $d$, and $\mathbf{y}_\mathrm{train}$ is a node-indexed vector of observed labels whose observed entries identify the labeled nodes $\mathcal{V}_\mathrm{train}\subset\mathcal{V}$, the goal of node classification is to predict the test-label vector $\mathbf{y}_\mathrm{test}$, whose entries are the labels of the remaining nodes $\mathcal{V}_\mathrm{test}=\mathcal{V}\setminus\mathcal{V}_\mathrm{train}$.

Following~\cite{nagler2023statistical}, we view the unknown distribution $p_0$ of a fully labeled graph as one draw from a prior $\Pi$. Sampling from this prior proceeds by drawing $p\sim\Pi$ and then drawing the fully labeled graph from $p$. A subset of these labels is observed and collected in $\mathbf{y}_{\mathrm{train}}$. The remaining labels form $\mathbf{y}_{\mathrm{test}}$. The posterior predictive distribution (PPD) is the conditional distribution of $\mathbf{y}_{\mathrm{test}}$ given $\mathcal{G}$ under this sampling scheme.

A prior-data fitted network~\cite{eremeev2026graphpfn,choi2026nodepfn} for node classification is a parametric conditional distribution $f_{\theta}(\mathbf{y}_{\mathrm{test}}\mid\mathcal{G})$ trained to approximate the PPD. The parameters minimize
\begin{equation}\label{eq:loss}
\mathcal{L}(\theta)=\mathbb{E}_{\Pi}\bigl[-\log f_{\theta}(\mathbf{y}_{\mathrm{test}}\mid\mathcal{G})\bigr],
\end{equation}
where the expectation follows the sampling scheme above. Nodes in $\mathcal{V}_{\mathrm{train}}$ are the context, and nodes in $\mathcal{V}_{\mathrm{test}}$ are the queries. At test time, $f_{\theta}(\cdot\mid\mathcal{G})$ is evaluated on a new graph in one forward pass, without gradient updates.

\section{Method}\label{sec:method}

\begin{figure*}
    \centering
    \includegraphics[width=\textwidth]{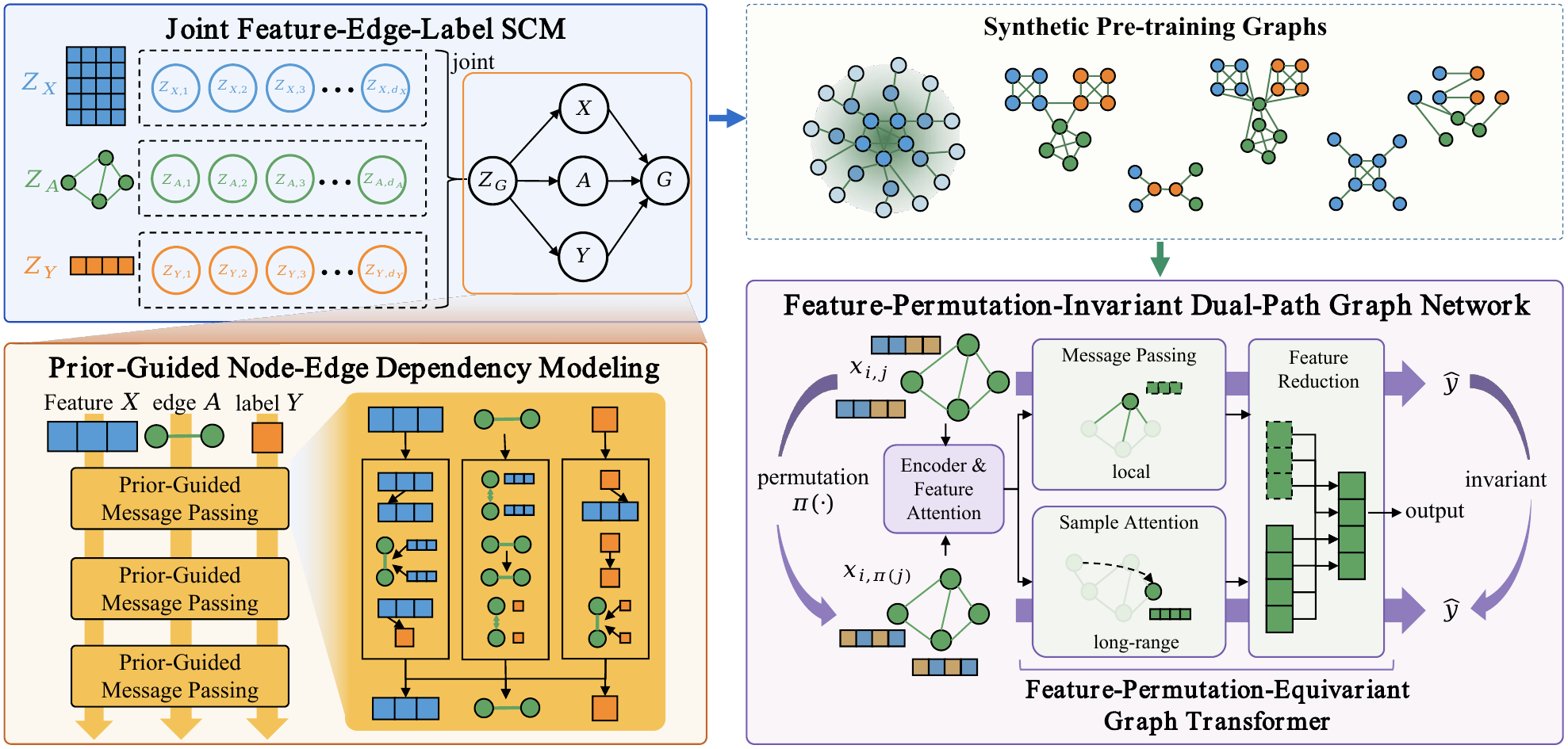}
    \caption{The framework of \name. Our proposed method is composed of three modules: (1) the \scm, which defines a causal prior over node features, graph topology, and labels; (2) the \scmdag~module, which generates graphs by updating node–edge variables at the sample level and individual dimensions at the feature level; and (3) the \architecture, which performs in-context prediction while addressing heterogeneous feature semantics and topological dependencies that vary in range.}
    \label{fig:framework}
\end{figure*}

In this section, we introduce \name~and its three components, as shown in Figure~\ref{fig:framework}. First, the \scm~defines a joint causal prior over node features, graph topology, and labels. Next, the \scmdag~module instantiates diverse graph generation strategies through multilayer causal updates of features, labels, and edges. Finally, the \architecture~supports in-context node classification.

\subsection{\SCM}

The module is designed to guide the joint generation of nodes and edges in synthetic graphs. Unlike the sequential graph generation strategies used in NodePFN~\cite{choi2026nodepfn} and GraphPFN~\cite{eremeev2026graphpfn}, which generate node features before edges or vice versa, our approach uses a graph causal mechanism to generate nodes and edges simultaneously.

Our method first defines graph-level latent variables for the generation mechanism. Specifically, we use a two-stage decoupling strategy to decompose the abstract graph latent variable into components that can be instantiated by the generation algorithm. First, considering node and edge information have distinct semantics, we separate the node-related latent variables from the edge-related latent variables and represent the edge-related variables as a vector of scalar components $Z_A=(Z_{A,1},\ldots,Z_{A,d_A})$, where $d_A$ denotes the number of scalar components in this group. Second, coupling node features with labels can weaken the causal relations between labels and other graph elements, particularly in a complex graph-generating SCM. Tabular PFNs primarily model causal dependencies between features and labels under an i.i.d. sample assumption. Graph generation, by contrast, must also capture dependencies within and across node and edge sets, resulting in a more complex SCM with potentially multi-hop causal pathways. This complexity makes it difficult to preserve the causal relationships between labels and other graph elements, particularly when label variables are coupled with a much larger set of node features. We therefore further decompose the node-related latent variables into node-feature latent variables $Z_X=(Z_{X,1},\ldots,Z_{X,d_X})$ and node-label latent variables $Z_Y=(Z_{Y,1},\ldots,Z_{Y,d_Y})$ where $d_X$ and $d_Y$ denote the numbers of scalar components in their respective groups.

Building on this two-stage decoupling, we next introduce the relations among the latent variables. The node-feature latent variables $Z_X$, the node-label latent variables $Z_Y$, and the edge latent variables $Z_A$ are collected into the graph-level latent variables $Z_G=(Z_X,Z_Y,Z_A)$. This latent variables determine the node features, the node labels, and the adjacency matrix:
\begin{equation}
    Z_G\rightarrow X, Z_G\rightarrow Y, Z_G\rightarrow A.
\end{equation}
These variables in turn determine the synthetic graph $G$:
\begin{equation}
    X\rightarrow G, Y\rightarrow G, A\rightarrow G.
\end{equation}

Then, we describe how the SCM samples a synthetic graph. First, we randomly instantiate the priors and draw initial values:
\begin{equation}\label{eq:initial}
    Z_G^0=(Z_X^0,Z_Y^0,Z_A^0).
\end{equation}
Existing PFN methods widely use a DAG to instantiate the causal relations between latent and observable variables. These initial values therefore serve only as the starting points for a DAG and do not constitute the complete set of latent variables used in the generation process. Second, the DAG propagates these inputs through its causal functions to generate additional variables:
\begin{equation}\label{eq:dag}
    V_G=\mathrm{DAG}(Z_G^0),
\end{equation}
where $\mathrm{DAG}(\cdot)$ denotes the composition of the causal functions along the DAG. These variables $V_G$ may represent additional latent variables or observable node features $X$, node labels $Y$, and the adjacency matrix $A$, depending on the outcome of the subsequent sampling step. During this process, the DAG is capable of capturing both causal relationships among variables within the same latent group, such as the relationship between $Z_{X,1}$ and $Z_{X,3}$, and relationships within and across the node and edge sets. Third, following the sampling protocols used in existing PFNs~\cite{eremeev2026graphpfn,choi2026nodepfn}, we randomly designate a subset of DAG-generated variables as the observed variables, while treating the remaining variables as latent. Together, the observed variables form the synthetic graph $G$.

\paragraph{Discussion}
Here, we explain how our SCM unifies and extends the SCM priors used by PFNs in both the tabular and graph domains. For tabular data~\cite{hollmanntabpfn}, our SCM reduces to a tabular SCM by omitting the latent groups $Z_Y$ and $Z_A$, along with the adjacency matrix $A$. For graph data, we consider GraphPFN~\cite{eremeev2026graphpfn} and NodePFN~\cite{choi2026nodepfn}. GraphPFN first generates graph topology from an SBM-family prior and then generates $X$ and $Y$ conditioned on the topology and $Z_X^0$. It can therefore be viewed as a special case with no separate latent group $Z_Y$ and with $Z_A$ corresponding directly to the adjacency matrix $A$. On the other hand, NodePFN first generates $X$ and $Y$ using a tabular SCM prior, then generates community-based graph topology conditioned on $Y$ using an SBM-family prior. It can be viewed as a special case with no separate latent groups $Z_Y$ and $Z_A$, in which $Z_X$ generates the entire graph $G$.

\subsection{\SCMDAG}
The SCM provides theoretical guidance for the synthetic graphs. This module further instantiates that guidance through diverse synthetic graph generation strategies. Consistent with Eq.~\eqref{eq:initial} and Eq.~\eqref{eq:dag}, the module is designed to serve as the DAG, which takes the initial latent state $Z_G^0$ as input and jointly generates variables governing node features, node labels, and edges. However, realizing this goal meets two challenges. First, causal dependencies between nodes in real-world graphs may span multiple hops, so one node can affect another that, in turn, influences more distant nodes. Second, existing graph PFNs, including GraphPFN and NodePFN, rely on SBM-family priors and fail to capture dependencies among edges beyond these factors. For example, link prediction methods often exploit triadic closure, whereby two nodes sharing a neighbor are more likely to be connected~\cite{wang2024neural}. Complex networks also show degree assortativity, whereby high-degree nodes preferentially connect to other high-degree nodes or to low-degree nodes~\cite{newman2002assortative}, and continuous core-periphery structure, whereby nodes have graded coreness and connections become denser toward the core~\cite{borgatti2000models}.

To address these challenges, we propose prior-guided message passing layers that jointly update the node- and edge-related variables. Multi-hop causal effects and dependencies among edges then evolve together during generation rather than the topology staying fixed given the community assignments. Let $V_G^\ell=(V_X^\ell,V_Y^\ell,V_A^\ell)$ denote the state at the $\ell$-th layer, where $V_X^\ell$, $V_Y^\ell$, and $V_A^\ell$ denote the feature-, label-, and edge-related variables, respectively, initialized as $V_G^0=Z_G^0=(Z_X^0,Z_Y^0,Z_A^0)$. Their status as observed or latent variables is determined by the sampling step that follows DAG propagation. The joint layerwise update is given by
\begin{equation}\label{eq:dag-layer}
(V_X^{\ell+1}, V_Y^{\ell+1}, V_A^{\ell+1}) = g(V_X^\ell, V_Y^\ell, V_A^\ell), 
\end{equation}
where $\ell$ denotes the layer index and $g(\cdot)$ maps the variable tuple at the $\ell$-th layer to the ($\ell+1$)-th layer.

After defining the layerwise update function $g(\cdot)$, we aim to capture sample-level causal relationships both within and across the node and edge sets, building on the second step of the \scm. To this end, we decompose $g(\cdot)$ into the following three functions: 
\begin{equation}\label{eq:layer-maps}
\begin{aligned}
V_X^{\ell+1} = g_X(V_X^\ell, V_Y^\ell, V_A^\ell),\\
V_Y^{\ell+1} = g_Y(V_X^\ell, V_Y^\ell, V_A^\ell),\\
V_A^{\ell+1} = g_A(V_X^\ell, V_Y^\ell, V_A^\ell),
\end{aligned} 
\end{equation}
where $g_X(\cdot)$ combines feature-to-feature and label-to-feature linear transformations with either standard or degree-based neighbor aggregation. $g_Y(\cdot)$ combines a feature-to-label linear transformation, a label residual, and label propagation with either standard or degree-based aggregation. $g_A(\cdot)$ combines a feature-similarity kernel, a label-similarity kernel or a Chung-Lu mechanism~\cite{chung2002connected}, and an adjacency residual. At each layer, all three variable groups are jointly updated based on the groups from the preceding layer, allowing the module to model these sample-level dependencies. The linear transformations in $g(\cdot)$, as in the random MLPs used in tabular settings~\cite{hollmanntabpfn}, allow the module to model causal relationships at the level of individual dimensions, complementing the sample-level dependencies captured by the joint updates.

At each layer, we randomly sample an activation function to introduce nonlinearity and inject random Gaussian noise into all three groups to improve the realism of the generated data. Finally, after the module completes its propagation in Eq.~\eqref{eq:dag}, we designate a subset of the generated variables $V_G=(V_X,V_Y,V_A)$ as the observables $X$, $Y$, and $A$:
\begin{equation}
    X \subseteq V_X,Y \subseteq V_Y,A \subseteq V_A,
\end{equation}
while retaining the remaining variables as latent, following the sampling protocol used in existing PFNs~\cite{eremeev2026graphpfn,choi2026nodepfn}. The pseudocode of this module is provided in Appendix~\ref{sec:pseudocode}. Although we initialize $V_G^0$ from a prior combining ER and Bernoulli SBM, our SCM-based graph generation strategy is more expressive than existing ER and Bernoulli SBM priors as follows.
\begin{theorem}[Graph Generation Strategy Expressiveness]
\label{thm:graphs}
Let $\mathcal{P}_{\mathrm{SCM}}$ be the set of distributions $p$ of graphs realizable by our SCM graph generation strategy, and let $\mathcal{P}_{\mathrm{ER}}$ and $\mathcal{P}_{\mathrm{SBM}}$ be the sets realizable by ER and Bernoulli SBM priors, respectively. Under mild assumptions, the SCM-based graph generation strategy is more expressive than the ER and Bernoulli SBM priors:
\begin{equation*}
\mathcal{P}_{\mathrm{ER}}\cup\mathcal{P}_{\mathrm{SBM}}\subsetneq\mathcal{P}_{\mathrm{SCM}}.
\end{equation*}
\end{theorem}
The proof is provided in Appendix~\ref{subsec:edge-independence}.

\subsection{\Architecture}\label{subsec:architecture}
Lastly, we design a backbone architecture to learn from the synthetic graphs while addressing two challenges. First, feature semantics are heterogeneous across both synthetic and real-world graphs, and the meaning of a given feature dimension may differ from one graph to another. These heterogeneous feature semantics limit the direct use of standard graph architectures such as GNNs~\cite{shchur2018pitfalls,velickovic2018graph} and graph Transformers~\cite{rampavsek2022recipe,chennagphormer}. They rely on order-sensitive inter-feature information exchange, implicitly assuming that a given feature dimension has a consistent meaning across graphs~\cite{lee2026view}. This assumption keeps them from being feature-permutation invariant, thereby making them ill-equipped to generalize across arbitrary graphs with varying features~\cite{finkelshtein2026equivariance}. Second, graphs vary in the range and strength of their topological dependencies. Some exhibit long-range dependencies that benefit from deeper message passing, whereas others involve weak neighbor influence, making extensive message passing unnecessary.

Inspired by LimiX~\cite{zhang2025limix}, our architecture addresses heterogeneous feature semantics across graphs with encoders for node features and labels, followed by feature-level and sample-level attention. Each node is encoded as $d$ feature embeddings and one label embedding, collected in $\mathbf{H}\in\mathbb{R}^{n\times(d+1)\times d_h}$. Feature-level attention acts independently on each node, 
\begin{equation}\label{eq:feature-attn}
\mathrm{FeatureAttn}(\mathbf{H}_i)=\mathrm{Attn}(\mathbf{H}_i,\mathbf{H}_i,\mathbf{H}_i),
\end{equation}
where $\mathbf{H}_i\in\mathbb{R}^{(d+1)\times d_h}$ and $\operatorname{Attn}(\cdot,\cdot,\cdot)$ denotes the attention function. For feature position $j$, let $\mathbf{S}=\mathbf{H}_{1:k,j}$ and $\mathbf{R}=\mathbf{H}_{k+1:n,j}$ denote the context and query embeddings. Sample-level attention is 
\begin{equation}\label{eq:sample-attn}
\mathrm{SampleAttn}(\mathbf{H}{:,j}) = \mathrm{Concat}\bigl( \mathrm{Attn}(\mathbf{S},\mathbf{S},\mathbf{S}),\mathrm{Attn}(\mathbf{R},\mathbf{S},\mathbf{S}) \bigr),
\end{equation}
where $\mathrm{Concat}(\cdot,\cdot)$ joins the two blocks along the node axis. The encoders and feature-level attention exchange information equivariantly across features. Sample-level attention maps context-label representations to the query-label representations used for prediction, which makes the network invariant to feature permutations.

To accommodate these varying ranges of topological dependencies, we extend the sample-level attention layer into a \gt~layer. Each \gt~layer runs an attention pathway and a message-passing pathway as two branches, then integrates them by feature cross-reduction. The attention pathway applies sample-level attention on every dimension:
\begin{equation}\label{eq:sa}
(\mathbf{H}_{i}^{\mathrm{sa}})_{j}
=
\mathrm{SampleAttn}(\mathbf{H}_{:,j})_{i},
\quad j=1,\ldots,d+1.
\end{equation}
The message-passing pathway propagates each feature along graph edges and does not read the label dimension:
\begin{equation}\label{eq:mp}
(\mathbf{H}_{i}^{\mathrm{mp}})_{j}
=
\mathrm{MP}(\mathbf{H}_{:,j},\mathbf{A})_{i},
\quad j=1,\ldots,d,
\end{equation}
where $\mathrm{MP}(\cdot)$ denotes the message-passing layer and is shared across features. Feature cross-reduction integrates the two pathways without changing the feature width:
\begin{equation}\label{eq:reduction}
\mathrm{FeatureReduction}(\mathbf{H}_i^{\mathrm{sa}},\mathbf{H}_i^{\mathrm{mp}})=\mathrm{Attn}(\mathbf{H}_i^{\mathrm{sa}},\mathbf{H}_i^{\mathrm{gt}},\mathbf{H}_i^{\mathrm{gt}}),
\end{equation}
where $\mathbf{H}_i^{\mathrm{gt}}$ is the feature-wise concatenation of the attention pathway $\mathbf{H}_i^{\mathrm{sa}}$ and the message-passing pathway $\mathbf{H}_i^{\mathrm{mp}}$. We use a dual-path message-passing and attention architecture rather than a sequential one such as GraphPFN because message-passing layers are limited by over-smoothing, over-squashing, and low expressivity with respect to the WL test~\cite{alon2021on,topping2022understanding}, and they may irreparably fail to preserve some information~\cite{rampavsek2022recipe}. Feature cross-reduction can leverage attention to capture information that is more useful for prediction from these two pathways.

We now summarize the complete architecture $f_\theta$. The model first encodes node features and observed labels, then processes the representations with a stack of blocks. Each block applies two feature-level attention stages in sequence, followed by a \gt. After the final block, an MLP classifier maps the query-label representations to class predictions. We show that the resulting architecture $f_\theta$ is invariant to permutations of the feature dimensions as follows. 
\begin{restatable}[Feature Permutation Invariance]{theorem}{featureinvariance}
\label{thm:feature-invariance}
Let $\pi$ be a permutation of the feature indices $\{1,\ldots,d\}$. 
Denote $\mathcal{G}_{\pi}=(\mathcal{V},\mathbf{A},\pi(\mathbf{X}),\mathbf{y}_{\mathrm{train}})$.
Then we have:
\begin{equation}
    f_{\theta}(\mathbf{y}_{\mathrm{test}}\mid\mathcal{G}_{\pi})=f_{\theta}(\mathbf{y}_{\mathrm{test}}\mid\mathcal{G}).
\end{equation}
\end{restatable}
The proof is provided in Appendix~\ref{subsec:feature-permutation}. 

Lastly, we introduce the overall training process. During pre-training, we sample a fully labeled graph from the prior $\Pi$ and partition its nodes into context nodes with observed labels $\mathbf{y}_\mathrm{train}$ and query nodes with held-out labels $\mathbf{y}_\mathrm{test}$. We withhold the query labels from the input and use them as prediction targets. We train $f_\theta$ to approximate the PPD by minimizing the expected negative log-likelihood in Eq.~\eqref{eq:loss}. At inference, we use the in-context learning setup, where the model predicts test node labels from the observed labels of the same graph in one forward pass, with its parameters held fixed. Its observed labels form the context, and the test nodes are the queries. We predict the label of each query node $i$ as the most probable class:
\begin{equation}
    \widehat{y}_i
    =
    \arg\max_{c}
    f_{\theta}
    \!\left(
        y_i=c
        \mid
        \mathcal{G}
        \right),
\end{equation}
where $c$ denotes a candidate class label.

\section{Experiments}\label{sec:exp}

\begin{table*}
\centering
\caption{Performance comparison of different methods on 23 real-world benchmark datasets, measured by accuracy. Numbers after ± indicate standard deviations. The best results are in \best{bold} and the second-best results are \secbest{underlined}.}
\label{tab:performance-comparison}
\begin{tabular}{l c c c c c c c c c}
\toprule
Dataset & GCN & GAT & GraphAny & TS-Mean & TS-GAT & RGVT & GraphPFN & NodePFN & \name\\
\midrule
Actor              & 28.55\scalebox{0.75}{±0.68} & 27.30\scalebox{0.75}{±0.22} & 28.60\scalebox{0.75}{±0.21} & 28.09\scalebox{0.75}{±0.93} & 28.80\scalebox{0.75}{±2.12} & 32.28\scalebox{0.75}{±0.68} & \secbest{34.46\scalebox{0.75}{±1.18}} & 32.82\scalebox{0.75}{±1.03} & \best{35.72\scalebox{0.75}{±1.18}}\\
AirBrazil          & 42.31\scalebox{0.75}{±7.98} & 57.69\scalebox{0.75}{±14.75} & 34.61\scalebox{0.75}{±16.09} & 39.23\scalebox{0.75}{±5.70} & 36.92\scalebox{0.75}{±10.03} & \best{86.16\scalebox{0.75}{±5.76}} & 67.69\scalebox{0.75}{±5.76} & \secbest{72.31\scalebox{0.75}{±6.62}} & \secbest{72.31\scalebox{0.75}{±2.88}}\\
AirEU              & 41.88\scalebox{0.75}{±3.60} & 32.50\scalebox{0.75}{±8.45} & 41.50\scalebox{0.75}{±6.50} & 35.88\scalebox{0.75}{±6.91} & 38.38\scalebox{0.75}{±6.32} & 54.75\scalebox{0.75}{±1.75} & \best{57.75\scalebox{0.75}{±4.87}} & 53.50\scalebox{0.75}{±4.53} & \secbest{56.50\scalebox{0.75}{±1.29}}\\
AirUS              & 46.49\scalebox{0.75}{±1.81} & 48.47\scalebox{0.75}{±4.17} & 43.64\scalebox{0.75}{±1.83} & 42.34\scalebox{0.75}{±2.12} & 41.69\scalebox{0.75}{±2.59} & 54.20\scalebox{0.75}{±1.68} & \secbest{61.87\scalebox{0.75}{±1.56}} & 59.14\scalebox{0.75}{±2.30} & \best{62.77\scalebox{0.75}{±1.63}}\\
Amazon-Comp        & \secbest{85.83\scalebox{0.75}{±0.86}} & \best{87.01\scalebox{0.75}{±0.50}} & 83.04\scalebox{0.75}{±1.24} & 81.37\scalebox{0.75}{±1.25} & 80.23\scalebox{0.75}{±2.44} & 83.85\scalebox{0.75}{±0.54} & 77.62\scalebox{0.75}{±1.40} & 78.44\scalebox{0.75}{±1.64} & 83.11\scalebox{0.75}{±1.57}\\
Amazon-Photo       & \secbest{91.88\scalebox{0.75}{±0.79}} & 91.86\scalebox{0.75}{±1.07} & 90.60\scalebox{0.75}{±0.82} & 90.18\scalebox{0.75}{±1.30} & 89.66\scalebox{0.75}{±1.23} & \best{91.98\scalebox{0.75}{±0.44}} & 87.89\scalebox{0.75}{±2.17} & 87.88\scalebox{0.75}{±1.32} & 91.63\scalebox{0.75}{±1.21}\\
Amazon-Ratings     & 47.35\scalebox{0.75}{±0.26} & 47.18\scalebox{0.75}{±0.42} & 42.74\scalebox{0.75}{±0.12} & 42.27\scalebox{0.75}{±1.40} & 41.92\scalebox{0.75}{±0.47} & 43.12\scalebox{0.75}{±0.97} & \secbest{48.52\scalebox{0.75}{±0.50}} & 44.47\scalebox{0.75}{±0.66} & \best{50.01\scalebox{0.75}{±0.22}}\\
Chameleon          & 41.31\scalebox{0.75}{±3.05} & 39.83\scalebox{0.75}{±2.10} & 32.72\scalebox{0.75}{±4.78} & 43.83\scalebox{0.75}{±3.99} & 36.40\scalebox{0.75}{±8.26} & 45.16\scalebox{0.75}{±2.18} & \secbest{46.81\scalebox{0.75}{±3.02}} & 46.68\scalebox{0.75}{±1.51} & \best{47.41\scalebox{0.75}{±3.07}}\\
Citeseer           & 63.40\scalebox{0.75}{±0.63} & \secbest{69.10\scalebox{0.75}{±1.59}} & 68.34\scalebox{0.75}{±0.23} & 68.66\scalebox{0.75}{±0.19} & 68.70\scalebox{0.75}{±0.48} & \best{70.08\scalebox{0.75}{±0.97}} & 60.78\scalebox{0.75}{±1.75} & 67.90\scalebox{0.75}{±0.25} & 68.86\scalebox{0.75}{±0.34}\\
Coauthor CS        & \secbest{91.83\scalebox{0.75}{±0.71}} & 88.47\scalebox{0.75}{±0.79} & 90.45\scalebox{0.75}{±0.59} & 90.92\scalebox{0.75}{±0.41} & 91.13\scalebox{0.75}{±0.43} & 90.12\scalebox{0.75}{±0.22} & 91.25\scalebox{0.75}{±0.42} & 90.92\scalebox{0.75}{±0.48} & \best{92.56\scalebox{0.75}{±0.36}}\\
Coauthor Physics   & \best{93.93\scalebox{0.75}{±0.37}} & 93.01\scalebox{0.75}{±0.89} & 92.69\scalebox{0.75}{±0.52} & 92.61\scalebox{0.75}{±0.61} & 92.23\scalebox{0.75}{±0.61} & 93.54\scalebox{0.75}{±0.06} & 91.12\scalebox{0.75}{±1.07} & 93.37\scalebox{0.75}{±0.45} & \secbest{93.77\scalebox{0.75}{±0.36}}\\
Cora               & 81.40\scalebox{0.75}{±0.70} & \secbest{81.70\scalebox{0.75}{±1.43}} & 79.38\scalebox{0.75}{±0.16} & 58.08\scalebox{0.75}{±0.48} & 22.42\scalebox{0.75}{±12.90} & 81.40\scalebox{0.75}{±0.96} & 75.98\scalebox{0.75}{±1.07} & \best{82.84\scalebox{0.75}{±0.28}} & 81.22\scalebox{0.75}{±0.07}\\
Cornell            & 35.14\scalebox{0.75}{±6.51} & 35.14\scalebox{0.75}{±3.52} & 65.94\scalebox{0.75}{±1.48} & 68.65\scalebox{0.75}{±2.42} & 71.35\scalebox{0.75}{±4.52} & 43.79\scalebox{0.75}{±2.65} & \secbest{71.89\scalebox{0.75}{±5.57}} & 69.73\scalebox{0.75}{±3.15} & \best{77.84\scalebox{0.75}{±2.02}}\\
DBLP               & 73.02\scalebox{0.75}{±2.22} & 73.87\scalebox{0.75}{±1.35} & 70.90\scalebox{0.75}{±0.88} & 66.42\scalebox{0.75}{±3.65} & 64.30\scalebox{0.75}{±3.75} & 75.19\scalebox{0.75}{±2.97} & 70.18\scalebox{0.75}{±1.63} & \best{78.08\scalebox{0.75}{±0.92}} & \secbest{77.88\scalebox{0.75}{±1.29}}\\
Deezer             & 53.69\scalebox{0.75}{±2.29} & \best{55.99\scalebox{0.75}{±3.78}} & 52.11\scalebox{0.75}{±2.79} & 52.31\scalebox{0.75}{±2.52} & 52.88\scalebox{0.75}{±2.86} & \secbest{54.56\scalebox{0.75}{±0.98}} & 50.79\scalebox{0.75}{±1.18} & 50.32\scalebox{0.75}{±0.97} & 50.02\scalebox{0.75}{±0.91}\\
Minesweeper        & \secbest{81.12\scalebox{0.75}{±0.37}} & 80.08\scalebox{0.75}{±0.04} & 80.30\scalebox{0.75}{±0.13} & 80.68\scalebox{0.75}{±0.38} & 80.72\scalebox{0.75}{±0.78} & 80.00\scalebox{0.75}{±0.00} & \best{83.62\scalebox{0.75}{±0.43}} & 79.85\scalebox{0.75}{±0.12} & 80.00\scalebox{0.75}{±0.00}\\
Pubmed             & 76.60\scalebox{0.75}{±0.32} & 77.30\scalebox{0.75}{±0.60} & 76.36\scalebox{0.75}{±0.17} & 74.98\scalebox{0.75}{±0.56} & 75.46\scalebox{0.75}{±0.86} & 78.26\scalebox{0.75}{±0.56} & 78.52\scalebox{0.75}{±0.75} & \secbest{78.64\scalebox{0.75}{±0.29}} & \best{80.10\scalebox{0.75}{±0.14}}\\
Questions          & 97.15\scalebox{0.75}{±0.04} & 97.11\scalebox{0.75}{±0.02} & 97.09\scalebox{0.75}{±0.02} & 97.02\scalebox{0.75}{±0.01} & 97.02\scalebox{0.75}{±0.01} & 97.02\scalebox{0.75}{±0.00} & \best{97.27\scalebox{0.75}{±0.03}} & 97.02\scalebox{0.75}{±0.01} & \secbest{97.17\scalebox{0.75}{±0.02}}\\
Squirrel           & 38.67\scalebox{0.75}{±1.84} & 38.78\scalebox{0.75}{±2.39} & 25.85\scalebox{0.75}{±1.11} & 39.80\scalebox{0.75}{±5.01} & 35.84\scalebox{0.75}{±3.33} & 41.99\scalebox{0.75}{±0.70} & \best{49.06\scalebox{0.75}{±1.20}} & 42.33\scalebox{0.75}{±1.73} & \secbest{44.87\scalebox{0.75}{±1.80}}\\
Texas              & 51.35\scalebox{0.75}{±2.71} & 54.05\scalebox{0.75}{±2.41} & 72.97\scalebox{0.75}{±2.71} & 73.51\scalebox{0.75}{±4.53} & 74.05\scalebox{0.75}{±4.10} & 69.73\scalebox{0.75}{±1.08} & \secbest{78.38\scalebox{0.75}{±10.95}} & 76.76\scalebox{0.75}{±3.67} & \best{83.24\scalebox{0.75}{±3.97}}\\
Tolokers           & \secbest{79.93\scalebox{0.75}{±0.10}} & 78.50\scalebox{0.75}{±0.55} & 78.18\scalebox{0.75}{±0.04} & 78.12\scalebox{0.75}{±0.09} & 78.01\scalebox{0.75}{±0.23} & 78.26\scalebox{0.75}{±0.20} & \best{83.75\scalebox{0.75}{±0.54}} & 78.18\scalebox{0.75}{±0.03} & 78.20\scalebox{0.75}{±0.05}\\
WikiCS             & 79.12\scalebox{0.75}{±0.45} & \secbest{79.27\scalebox{0.75}{±0.20}} & 74.95\scalebox{0.75}{±0.61} & 74.16\scalebox{0.75}{±2.07} & 73.24\scalebox{0.75}{±2.49} & 78.34\scalebox{0.75}{±0.57} & 78.75\scalebox{0.75}{±0.92} & 75.98\scalebox{0.75}{±0.80} & \best{81.07\scalebox{0.75}{±0.70}}\\
Wisconsin          & 37.25\scalebox{0.75}{±1.64} & 52.94\scalebox{0.75}{±3.10} & 65.10\scalebox{0.75}{±3.22} & 61.18\scalebox{0.75}{±11.38} & 65.88\scalebox{0.75}{±9.86} & 61.18\scalebox{0.75}{±5.74} & \secbest{82.35\scalebox{0.75}{±6.44}} & 75.69\scalebox{0.75}{±3.42} & \best{83.14\scalebox{0.75}{±3.64}}\\
\midrule
\multicolumn{1}{c}{Average Accuracy} & 63.44 & 64.66 & 64.70 & 64.36 & 62.49 & 68.91 & \secbest{70.71} & 70.12 & \best{72.58}\\
\multicolumn{1}{c}{Average Ranking} & 4.74 & 5.07 & 6.41 & 6.63 & 6.67 & 4.30 & \secbest{3.87} & 4.70 & \best{2.61}\\
\bottomrule
\end{tabular}
\end{table*}

In this section, we conduct extensive experiments to evaluate \name, focusing on four research questions: 
\begin{itemize}[leftmargin=1.5em]
    \item \textbf{RQ1}: Can \name~effectively handle multi-domain graph tasks without utilizing real data during pre-training? %
    \item \textbf{RQ2}: How do the number of observed labels and label noise affect the performance of \name?
    \item \textbf{RQ3}: How do the key components affect the results?
    \item \textbf{RQ4}: Can \name~generate diverse synthetic graphs with topological patterns? %
\end{itemize}

\subsection{Experimental Setup}

\paragraph{Datasets}

Following~\cite{choi2026nodepfn}, we evaluate our method on 23 benchmark datasets spanning a wide range of domains, including academic collaboration, social, e-commerce, and airline route networks. The task is node classification and the datasets contain between 2 and 15 classes. Further details are provided in Appendix~\ref{subsec:dataset}.

\paragraph{Baselines}

We compare \name~against baselines from three groups. The first comprises standard GNNs, including GCN~\cite{kipf2017semi} and GAT~\cite{velickovic2018graph}, trained from scratch on each dataset. The second comprises conventional GFMs, including GraphAny~\cite{zhao2025graphany}, TS-GNN~\cite{finkelshtein2026equivariance}, and RGVT~\cite{lee2026view}. The third comprises PFN-based GFMs, including GraphPFN~\cite{eremeev2026graphpfn} and NodePFN~\cite{choi2026nodepfn}. To ensure a fair comparison, we adapt both GFM groups using only the training split of each target dataset, with hyperparameters fixed across datasets.

\paragraph{Implementation Details}

We realize \name~as a 9-layer transformer with embedding dimension 96, three attention heads of dimension 32, and four position-wise feed forward networks of hidden size 384 with GELU activations in each layer. Every sublayer uses pre-normalization with a LayerNorm that has no affine parameters. A linear head maps the test-node label representations to 20 class logits. Optimization uses AdamW with learning rate $3\times10^{-4}$, no weight decay, gradient clipping at 1, and a cosine schedule that warms up for 5 epochs and decays to $10^{-5}$. Mixed precision is enabled. We train for 50 epochs on 128 GPUs, with one graph per GPU and 1000 steps per epoch, and use the epoch-50 checkpoint. Only the training split is provided as labeled context; validation and test nodes are queries, and their labels are masked. Edges are symmetrized and self-loops are removed. We follow the fixed split of each dataset, repeat the evaluation for five runs, and report mean accuracy and standard deviation. Additional implementation details are provided in Appendix~\ref{subsec:imple-detail}.

\subsection{Performance Comparisons}

To answer \textbf{RQ1}, we report the results of \name~compared to baseline methods in Table~\ref{tab:performance-comparison}. The results show that \name~outperforms standard GNNs, conventional GFMs, and PFN-based GFMs, achieving the best average accuracy and average ranking. 
For example, on Cornell, \name~achieves an 8.28\% relative improvement in accuracy over the second-best method. Similarly, on Texas, \name~achieves a 6.20\% relative improvement in accuracy over the second-best method. Compared with GraphPFN, the strongest baseline by average accuracy, \name~improves average accuracy by 2.64\% in relative terms and reduces the average rank by 1.26. Beyond Cornell and Texas, it also leads on homophilic graphs such as Pubmed and Coauthor CS.

Furthermore, the PFN-based GFMs achieve the top two results in both average accuracy and average rank, highlighting the promise of PFN-based methods for multi-domain node classification. RGVT is the strongest baseline outside this group, yet its average accuracy of 68.91 still trails \name~by 3.67 points. This performance may stem from their synthetic-prior-based in-context learning mechanism, which enables adaptation to individual graph datasets at inference time without additional training or dataset-specific hyperparameter tuning. By leveraging labeled context nodes, \name~adapts its predictions to downstream node-classification tasks, yielding generalizable and flexible performance.

\subsection{In-Depth Analyses}

\subsubsection{Effectiveness under Varying Numbers of Observed Labels and Label Noise}

\begin{figure}
    \centering
    \begin{subfigure}{\linewidth}
        \centering
        \includegraphics[width=\linewidth]{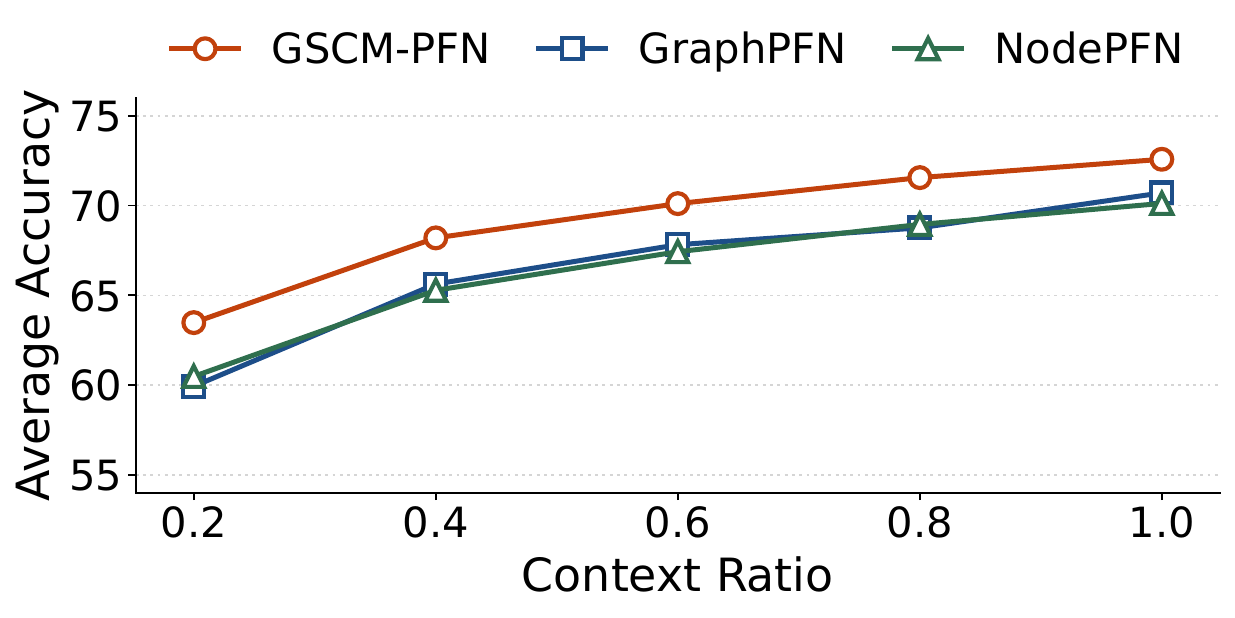}
    \end{subfigure}
    \begin{subfigure}{\linewidth}
        \centering
        \includegraphics[width=\linewidth]{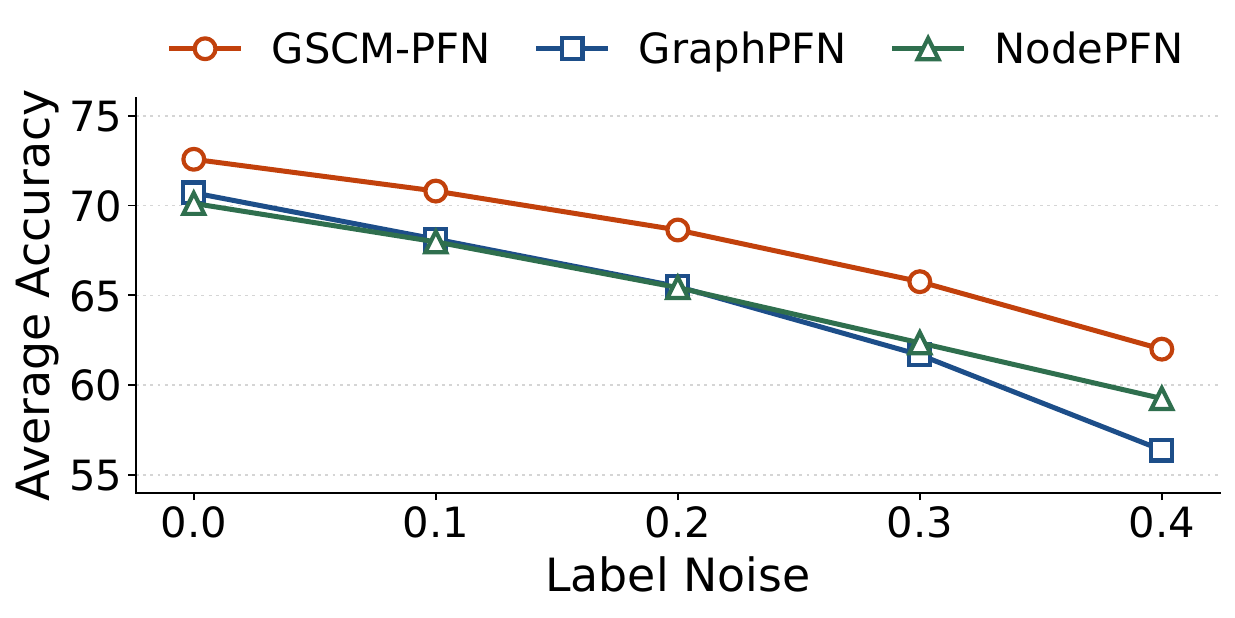}
    \end{subfigure}
    \caption{The impact of context ratio and label noise.}
    \label{fig:context}
\end{figure}

To answer \textbf{RQ2}, we evaluate \name, GraphPFN, and NodePFN while varying the number of observed labels and the level of label noise, and report average accuracy in Figure~\ref{fig:context}. We keep 20\% to 100\% of the original training split as observed labels and randomly corrupt 0\% to 40\% of them. The results show that \name's performance improves as more labels are observed and fewer of them are corrupted, suggesting that a larger and more reliable set of observed labels helps the model capture the properties of the target graph. Moreover, \name~consistently achieves higher average accuracy than GraphPFN and NodePFN across these settings, demonstrating that the improvement remains stable as the number and noise of observed labels vary.

\subsubsection{Ablation Studies}

\begin{table}
\centering
\caption{Results of ablation for the synthetic prior and the model architecture. We report the average test accuracy. Numbers after $-$ indicate the change relative to the full model.}
\label{tab:ablation}
\begin{tabular}{l c}
\toprule
Method & Accuracy \\ 
\midrule
Full \name &\best{72.58} \\
\midrule
\rowcolor{gray!15}\multicolumn{2}{c}{\textit{Prior}} \\
\text{w/o prior-guided message passing} & 68.00\scalebox{0.75}{($-$4.58)} \\
\text{w/o $Y$ Decouple} & 71.62\scalebox{0.75}{($-$0.96)} \\
\text{w/o $A$ Decouple} & 71.25\scalebox{0.75}{($-$1.33)} \\
\text{w/o multilayer structure} & 71.11\scalebox{0.75}{($-$1.47)} \\
\text{w/o feature relationships} & 71.06\scalebox{0.75}{($-$1.52)} \\
\text{w/o sample relationships} & 69.47\scalebox{0.75}{($-$3.11)} \\
\midrule
\rowcolor{gray!15}\multicolumn{2}{c}{\textit{Architecture}} \\
\text{w/o feature cross-reduction} & 72.09\scalebox{0.75}{($-$0.49)} \\
\text{w/o message passing} & 68.67\scalebox{0.75}{($-$3.92)} \\
\bottomrule
\end{tabular}
\end{table}

To answer \textbf{RQ3}, we ablate six components of the synthetic prior and two components of the model architecture. For the synthetic prior, \textit{w/o prior-guided message passing} removes the \scmdag~module. \textit{w/o $Y$ Decouple} removes the separate latent group for node labels, and \textit{w/o $A$ Decouple} removes the separate latent group for edges. \textit{w/o multilayer structure} replaces the multiple \scmdag~layers with a single layer. \textit{w/o feature relationships} removes the linear transformations that model causal dependencies among variables within each group. \textit{w/o sample relationships} removes the joint layerwise updates of the feature-, label-, and edge-related variables, which capture sample-level dependencies. For the architecture, \textit{w/o feature cross-reduction} replaces feature cross-reduction with the mean of the attention and message-passing pathways. \textit{w/o message passing} removes feature-wise message passing along graph edges.

Table~\ref{tab:ablation} reports the average test accuracy, with drops measured relative to the full model. The \scmdag~module is the most important prior component for generating diverse synthetic graphs with structured dependencies, and removing it causes the largest drop. Explicitly modeling label- and edge-related latent variables benefits the prior, as shown by \textit{w/o $Y$ Decouple} and \textit{w/o $A$ Decouple}. Multilayer causal updates help capture dependencies across hops, and \textit{w/o multilayer structure} lowers accuracy. Feature-level relationships improve the prior, while sample-level dependencies contribute more, with \textit{w/o sample relationships} dropping further than \textit{w/o feature relationships}. On the architecture, message passing makes a substantial contribution to capturing graph topology, and feature cross-reduction adds a smaller but consistent benefit by combining the two pathways. 

Overall, each component of the synthetic prior and the model architecture contributes to \name's performance, with the \scmdag~module and the message-passing pathway showing the largest effects.

\subsubsection{Characteristics of Synthetic Prior Datasets}

\begin{figure}
    \centering
        \begin{subfigure}{0.49\linewidth}
        \centering
        \includegraphics[width=\linewidth]{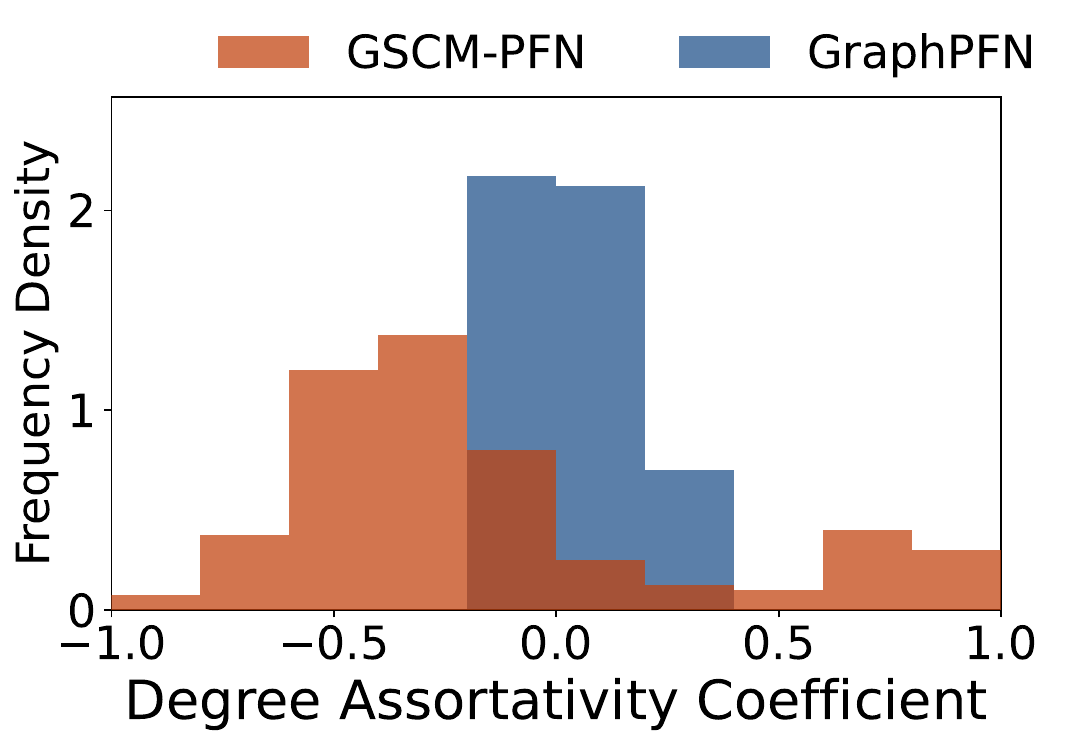}
        \end{subfigure}
        \begin{subfigure}{0.49\linewidth}
        \centering
        \includegraphics[width=\linewidth]{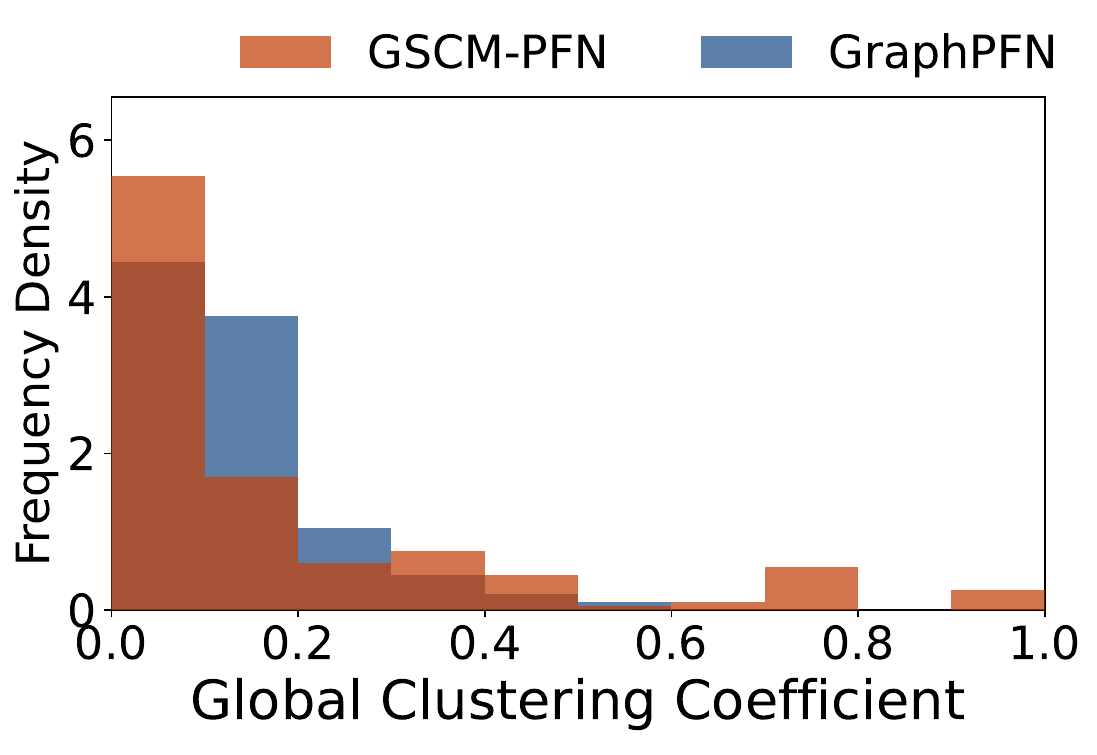}
        \end{subfigure}
    \centering
        \begin{subfigure}{0.49\linewidth}
        \centering
        \includegraphics[width=\linewidth]{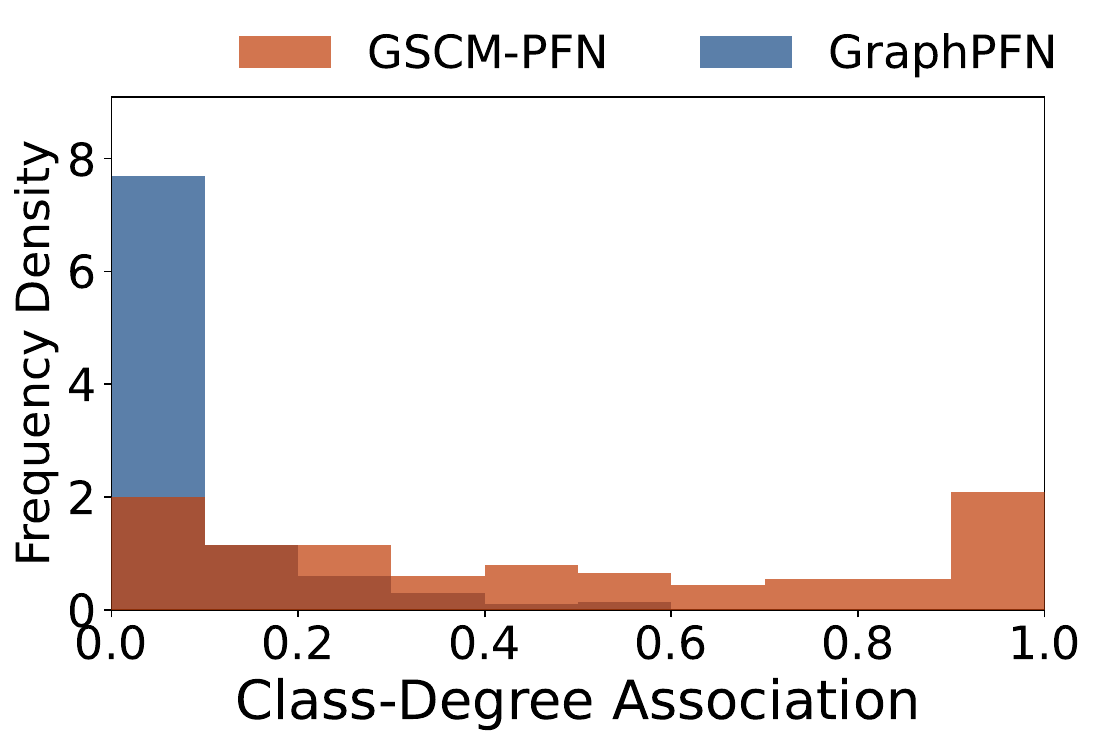}
        \end{subfigure}
        \begin{subfigure}{0.49\linewidth}
        \centering
        \includegraphics[width=\linewidth]{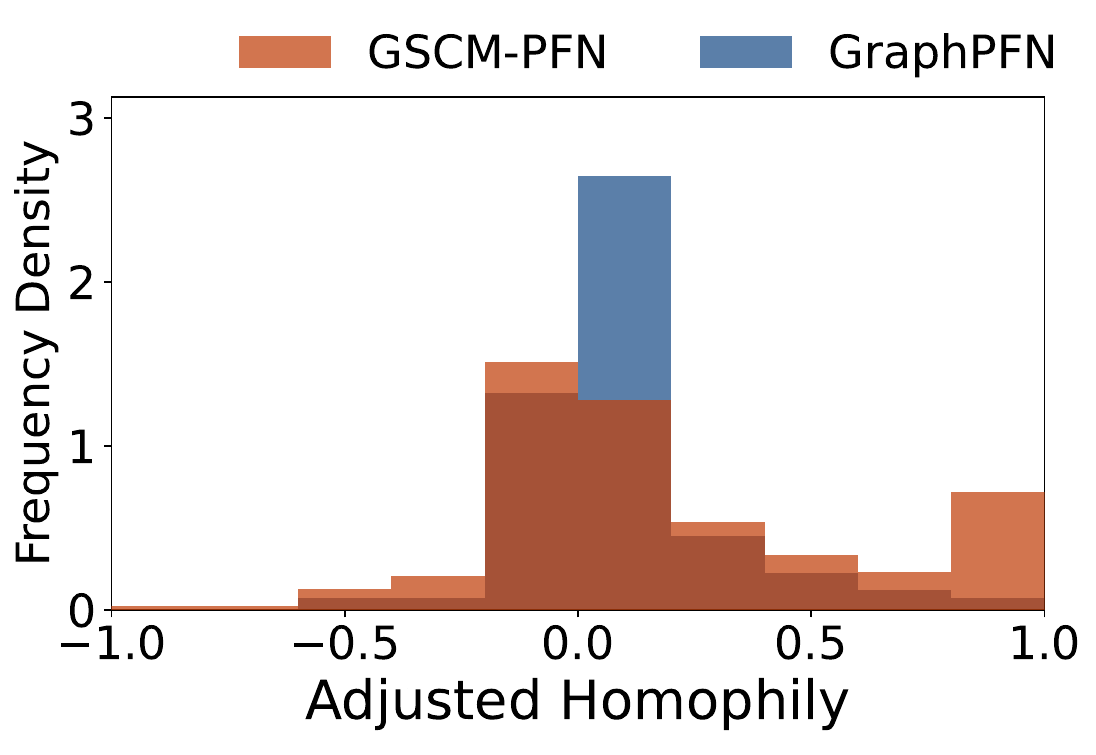}
        \end{subfigure}
    \caption{Histograms comparing the degree assortativity, global clustering, class-degree association, and adjusted homophily of synthetic graphs.}
    \label{fig:statistics}
\end{figure}

\begin{figure*}
    \centering
    \begin{subfigure}[t]{0.32\textwidth}
        \centering
        \includegraphics[width=0.43\linewidth,height=0.12\textheight,keepaspectratio]{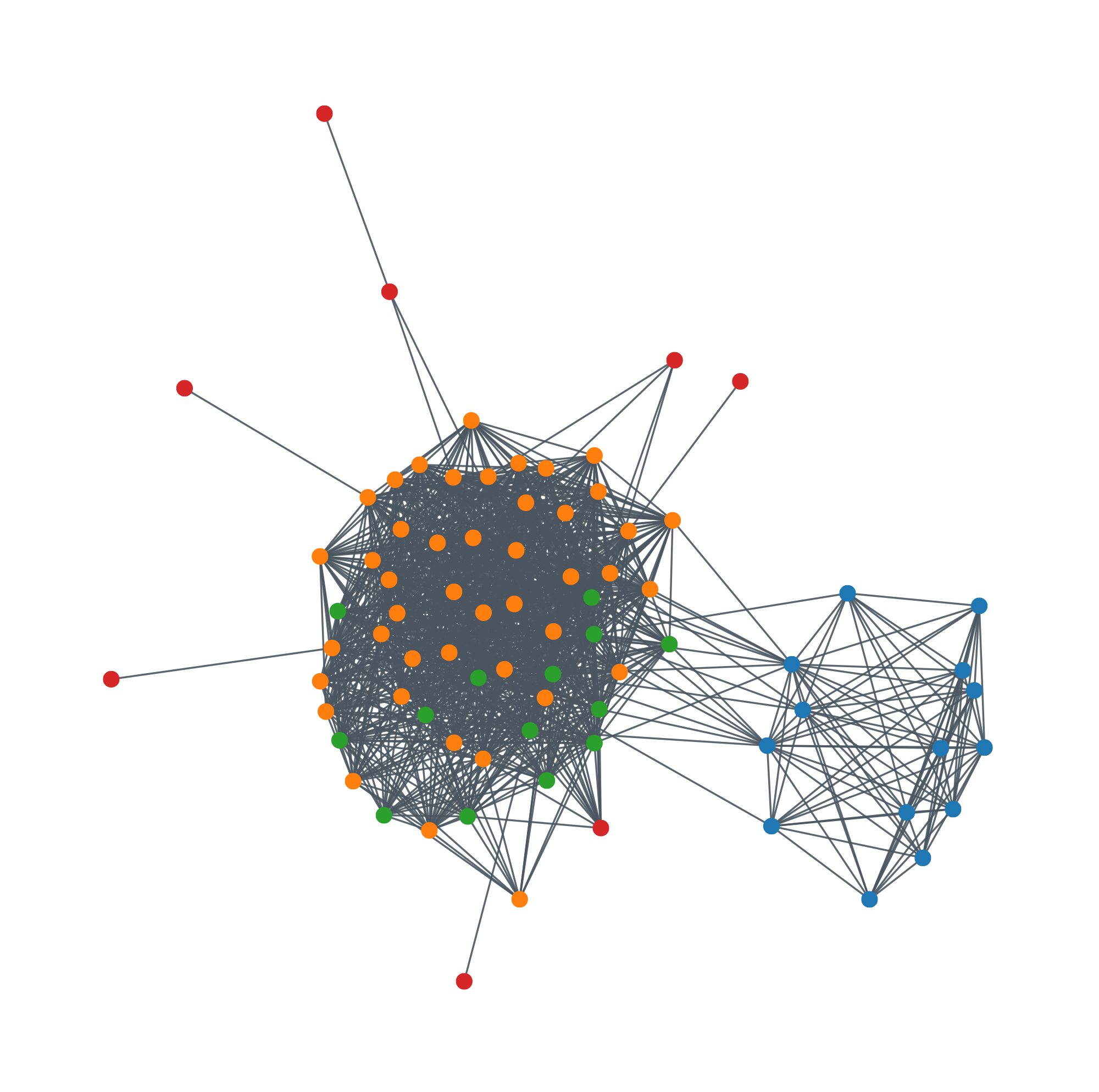}\hfill
        \includegraphics[width=0.43\linewidth,height=0.12\textheight,keepaspectratio]{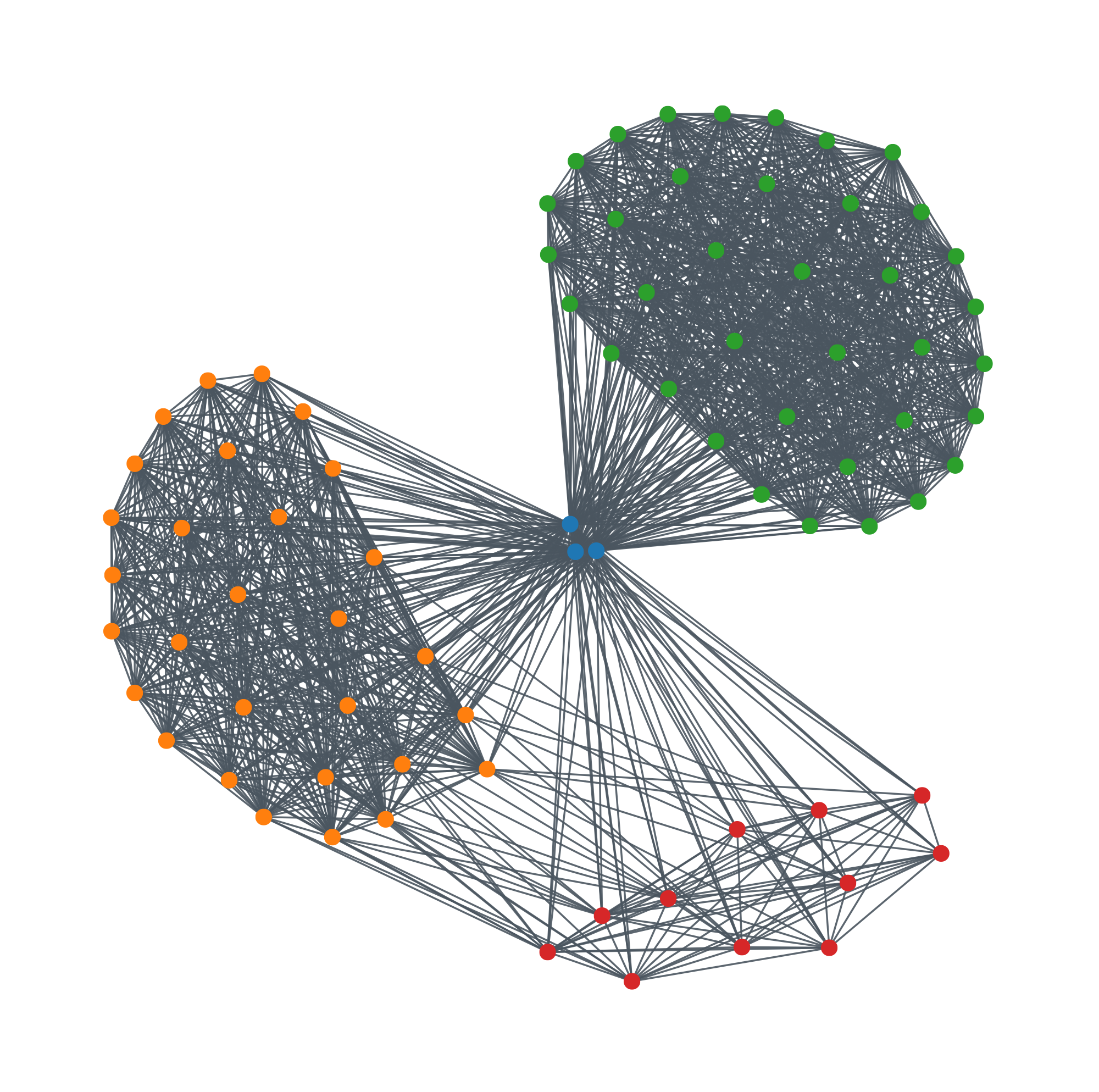}\\[0.5em]
        \includegraphics[width=0.43\linewidth,height=0.12\textheight,keepaspectratio]{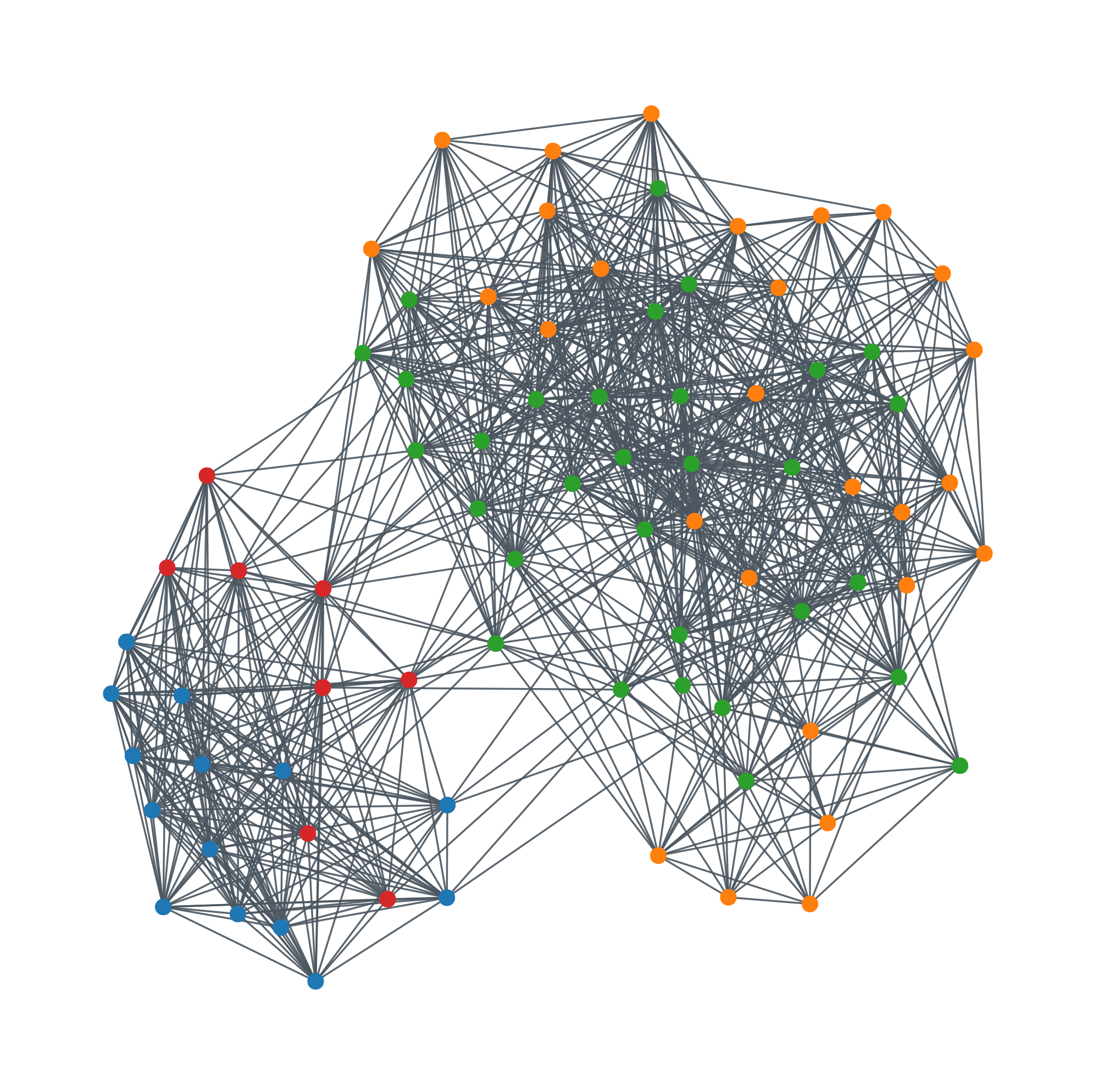}\hfill
        \includegraphics[width=0.43\linewidth,height=0.12\textheight,keepaspectratio]{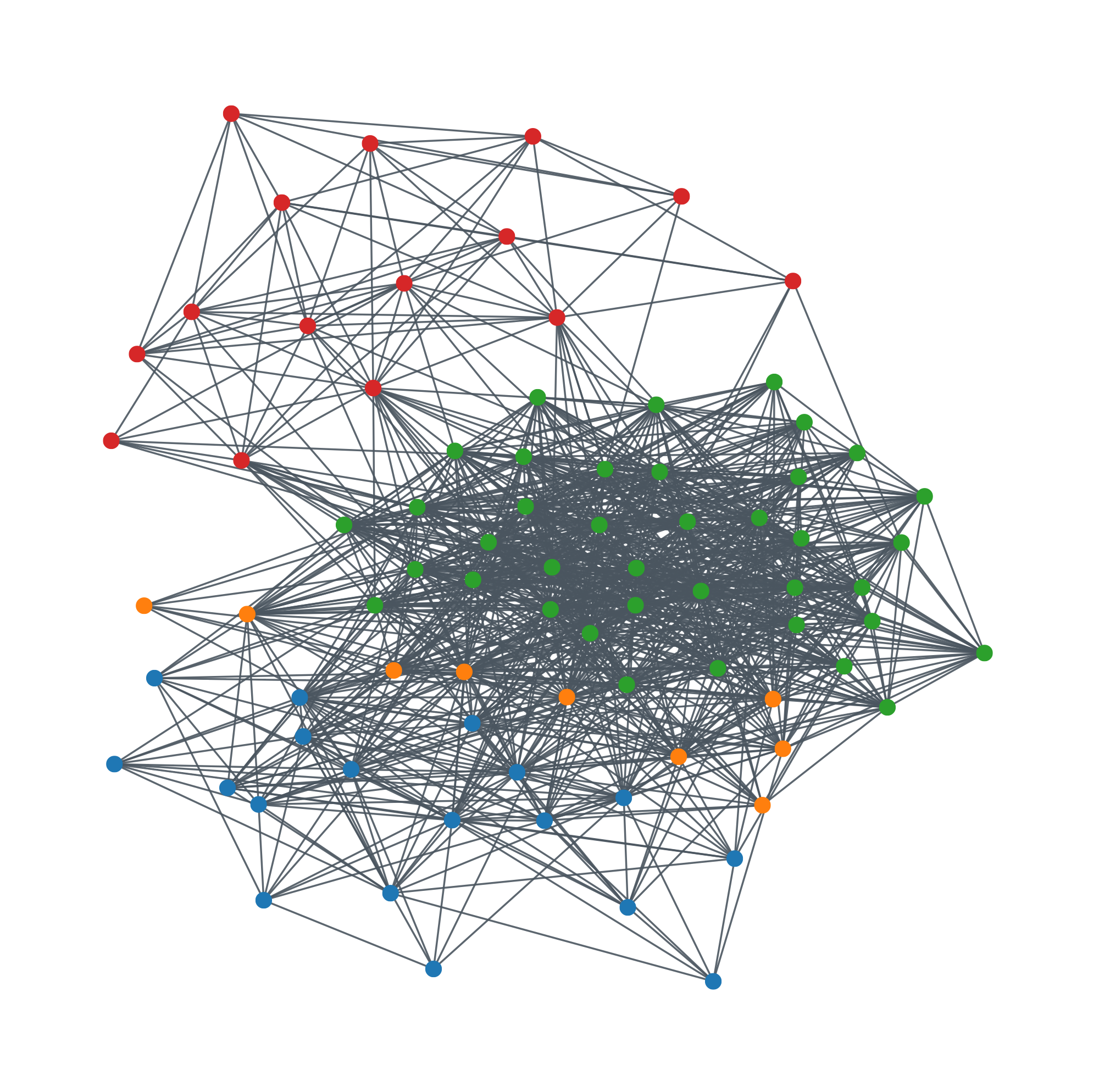}
    \end{subfigure}\hfill
    \begin{subfigure}[t]{0.32\textwidth}
        \centering
        \includegraphics[width=0.43\linewidth,height=0.12\textheight,keepaspectratio]{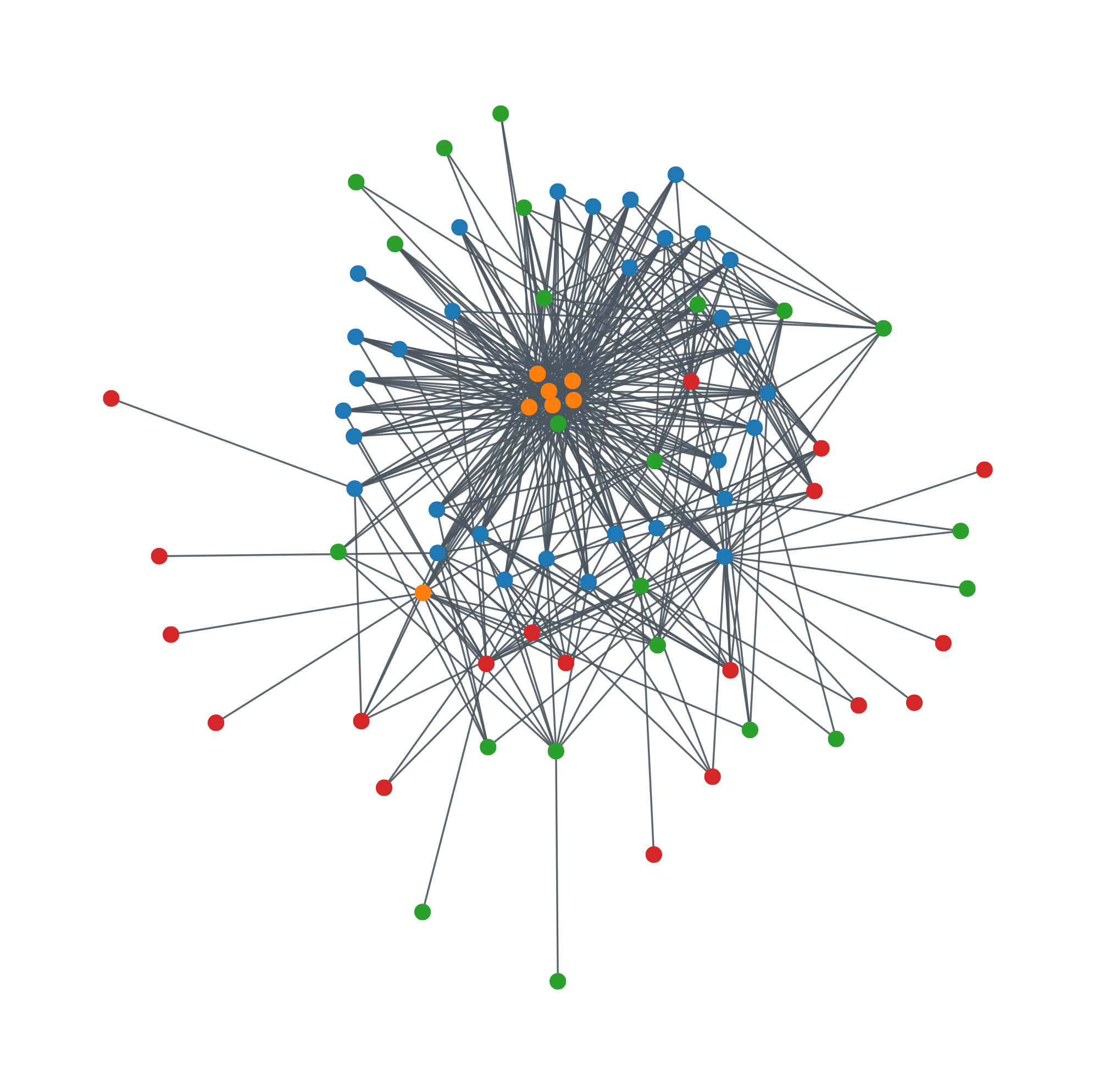}\hfill
        \includegraphics[width=0.43\linewidth,height=0.12\textheight,keepaspectratio]{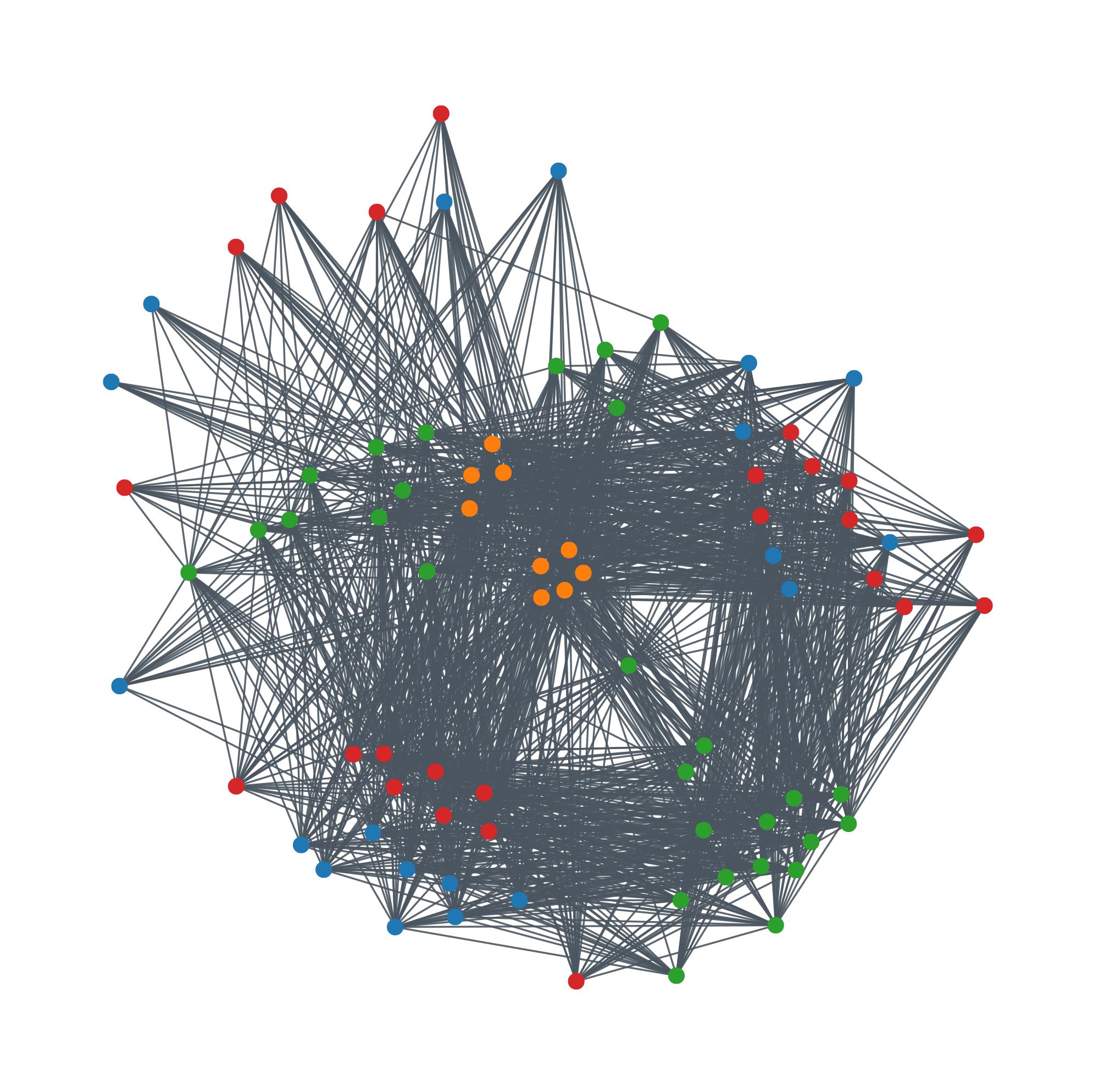}\\[0.5em]
        \includegraphics[width=0.43\linewidth,height=0.12\textheight,keepaspectratio]{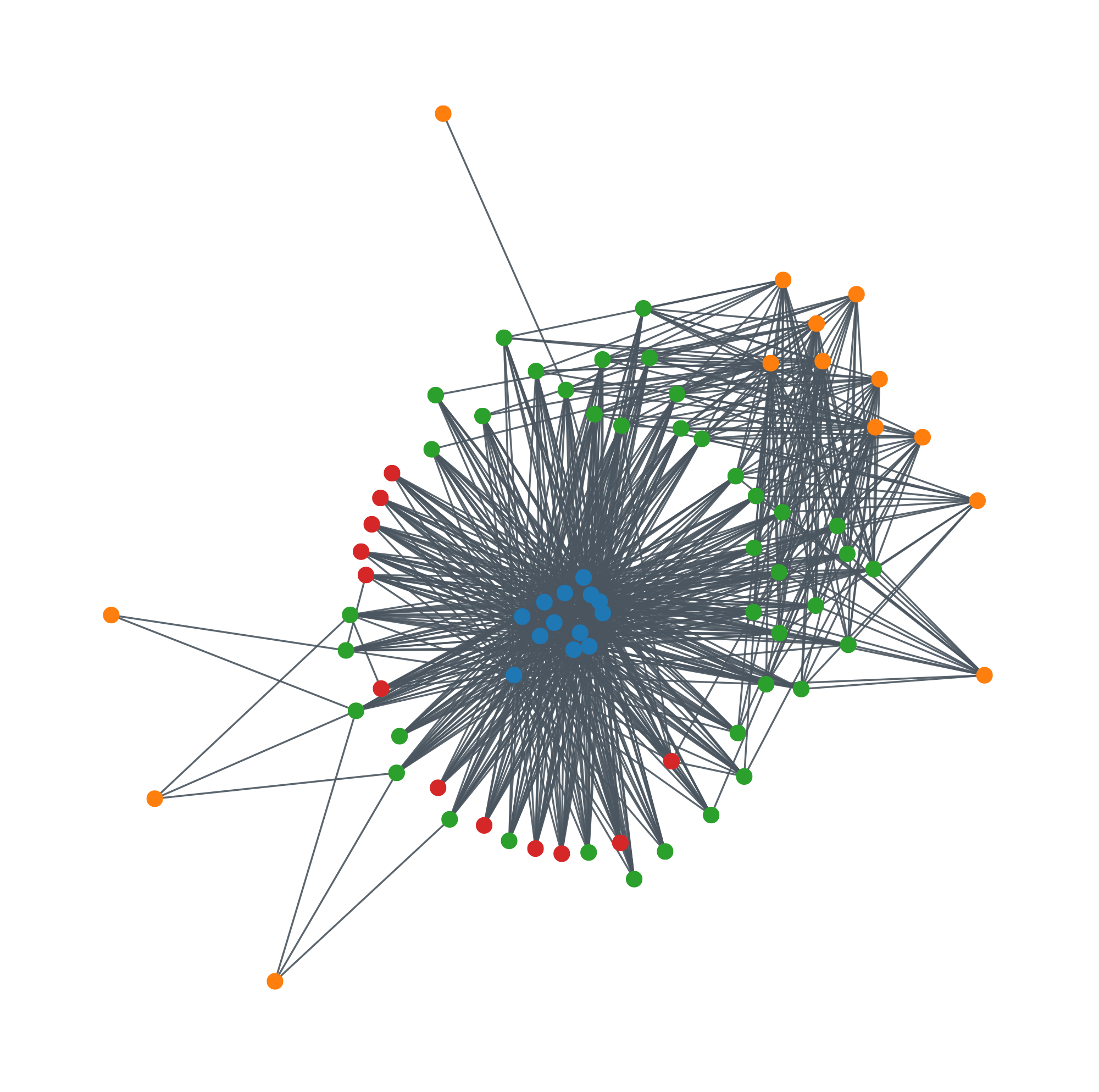}\hfill
        \includegraphics[width=0.43\linewidth,height=0.12\textheight,keepaspectratio]{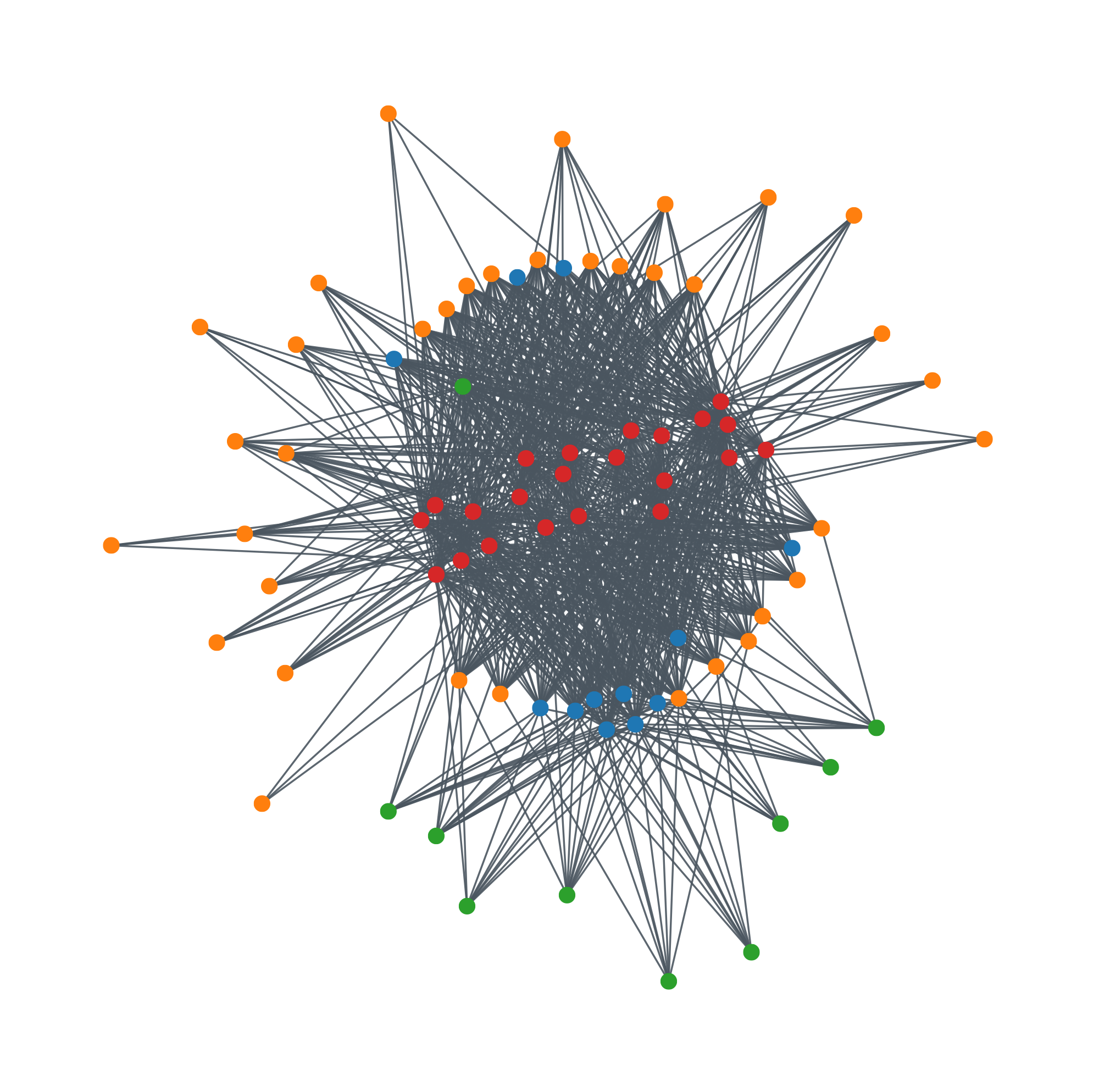}
    \end{subfigure}\hfill
    \begin{subfigure}[t]{0.32\textwidth}
        \centering
        \includegraphics[width=0.43\linewidth,height=0.12\textheight,keepaspectratio]{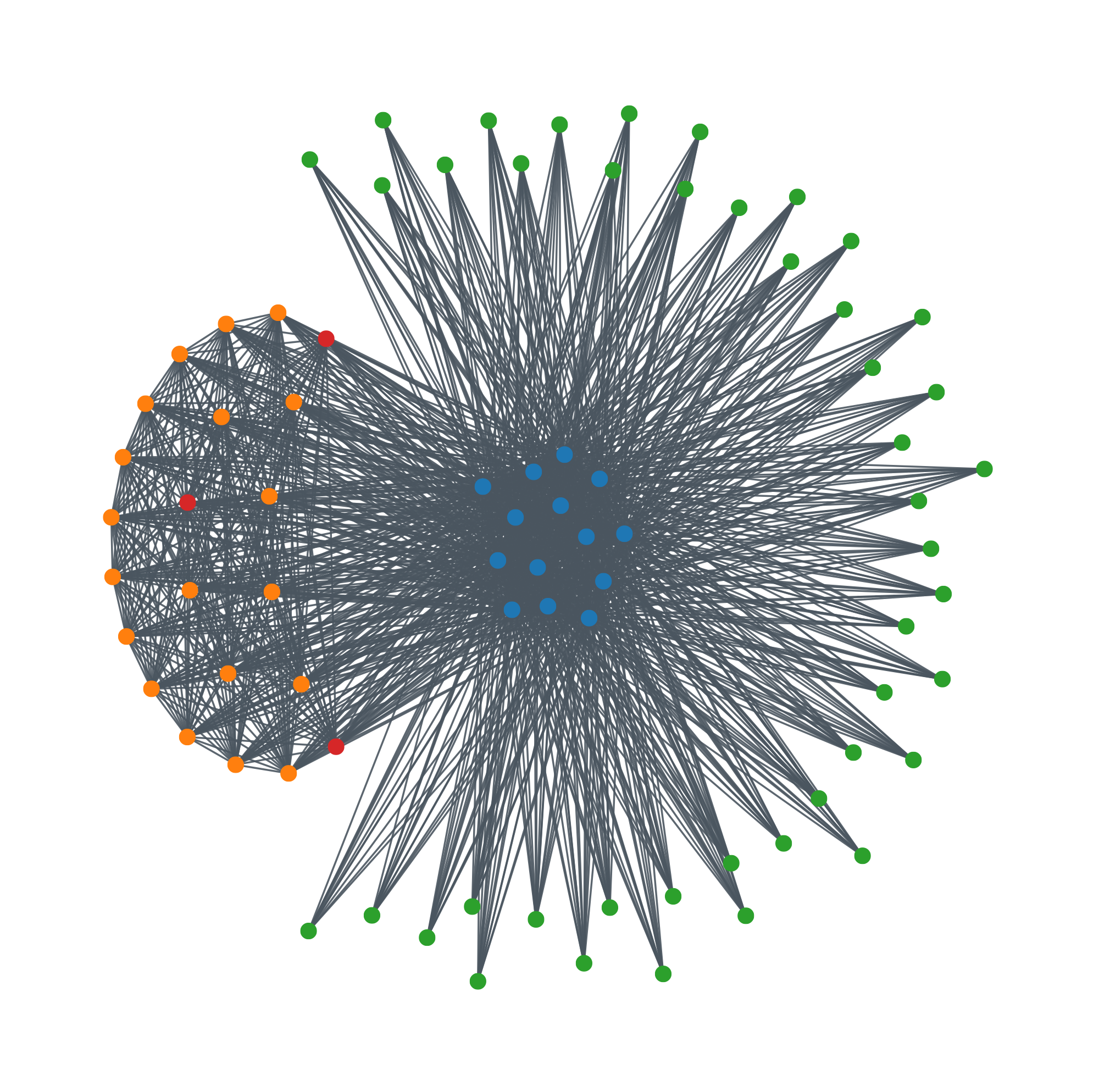}\hfill
        \includegraphics[width=0.43\linewidth,height=0.12\textheight,keepaspectratio]{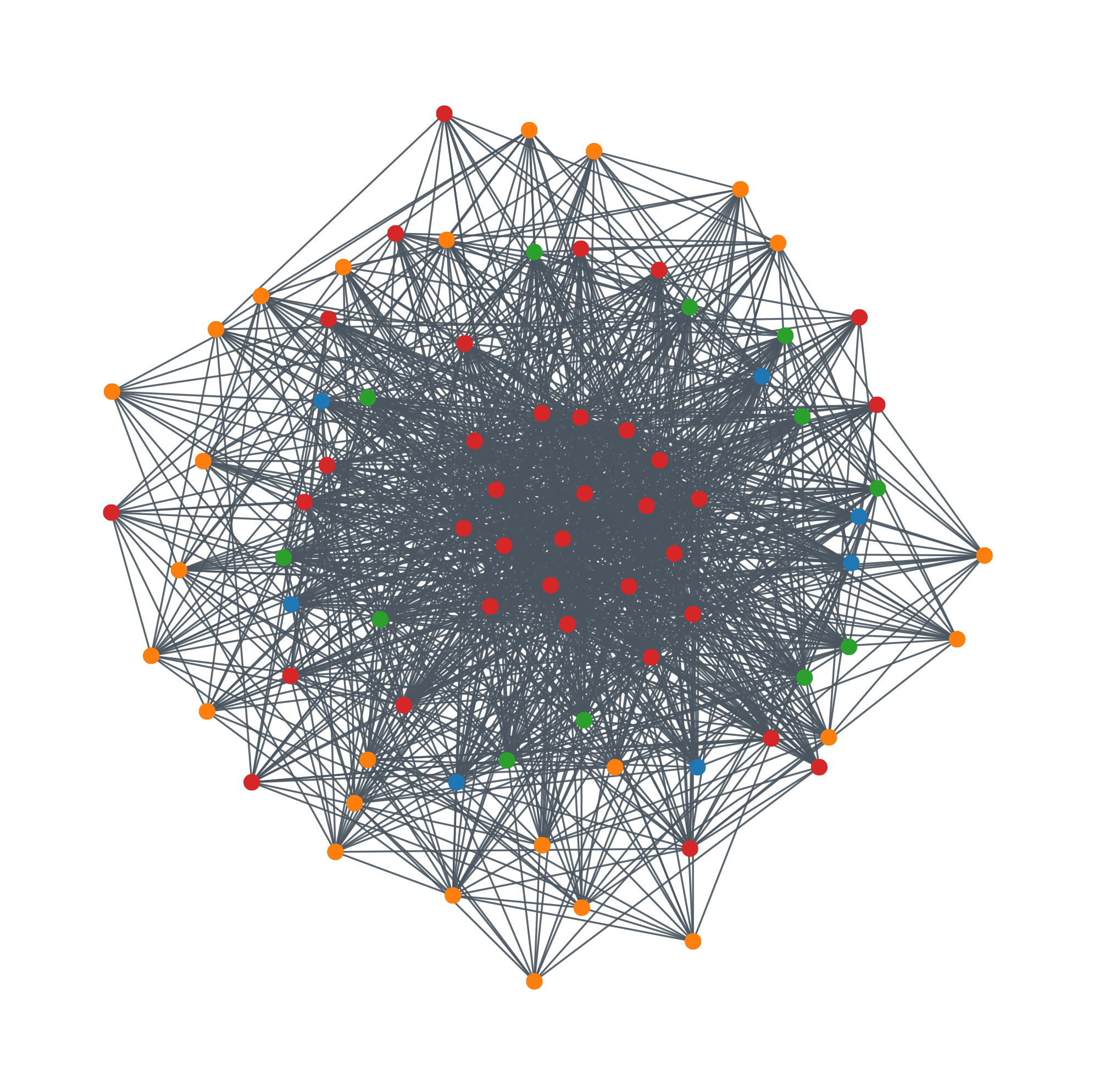}\\[0.5em]
        \includegraphics[width=0.43\linewidth,height=0.12\textheight,keepaspectratio]{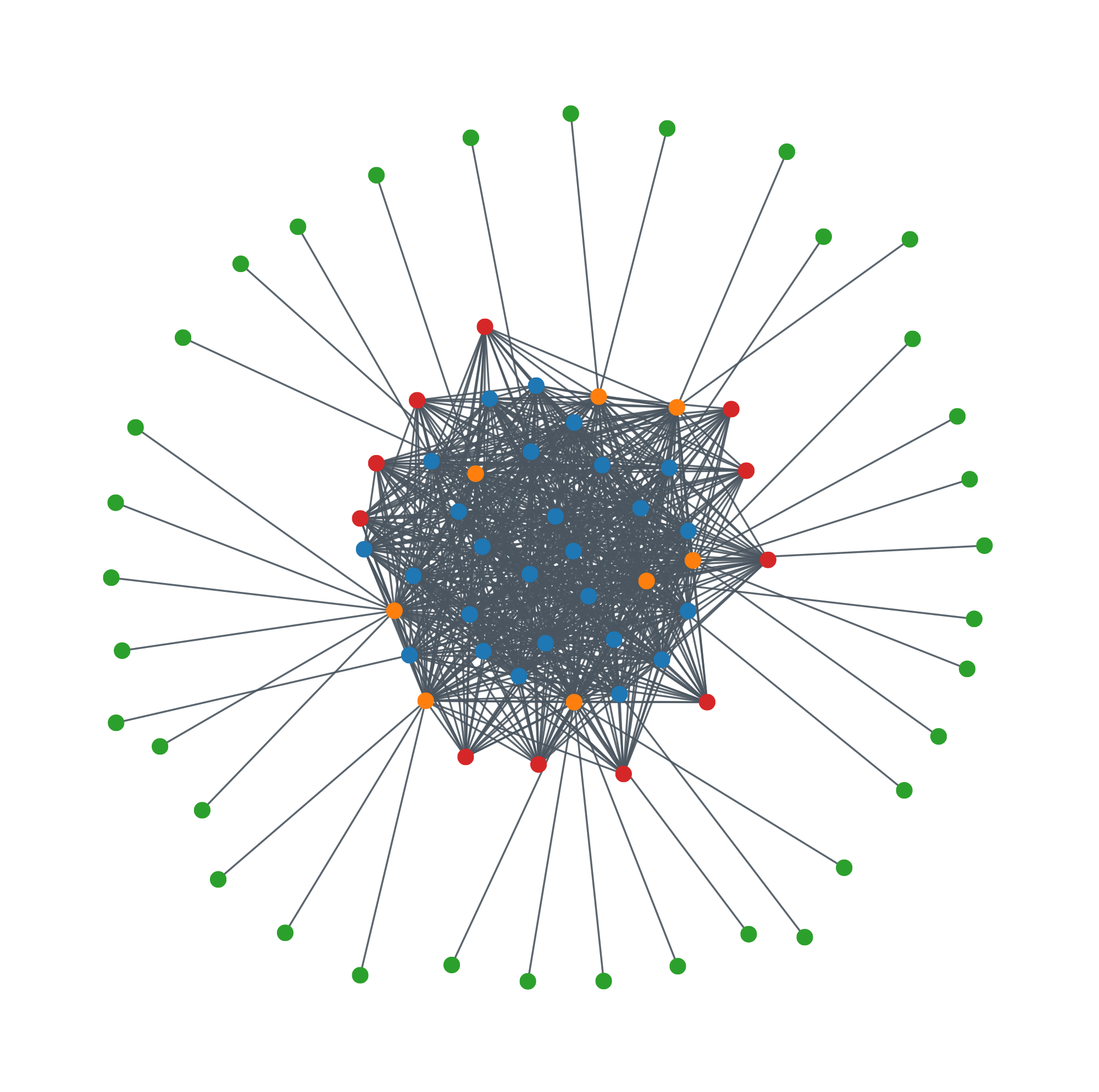}\hfill
        \includegraphics[width=0.43\linewidth,height=0.12\textheight,keepaspectratio]{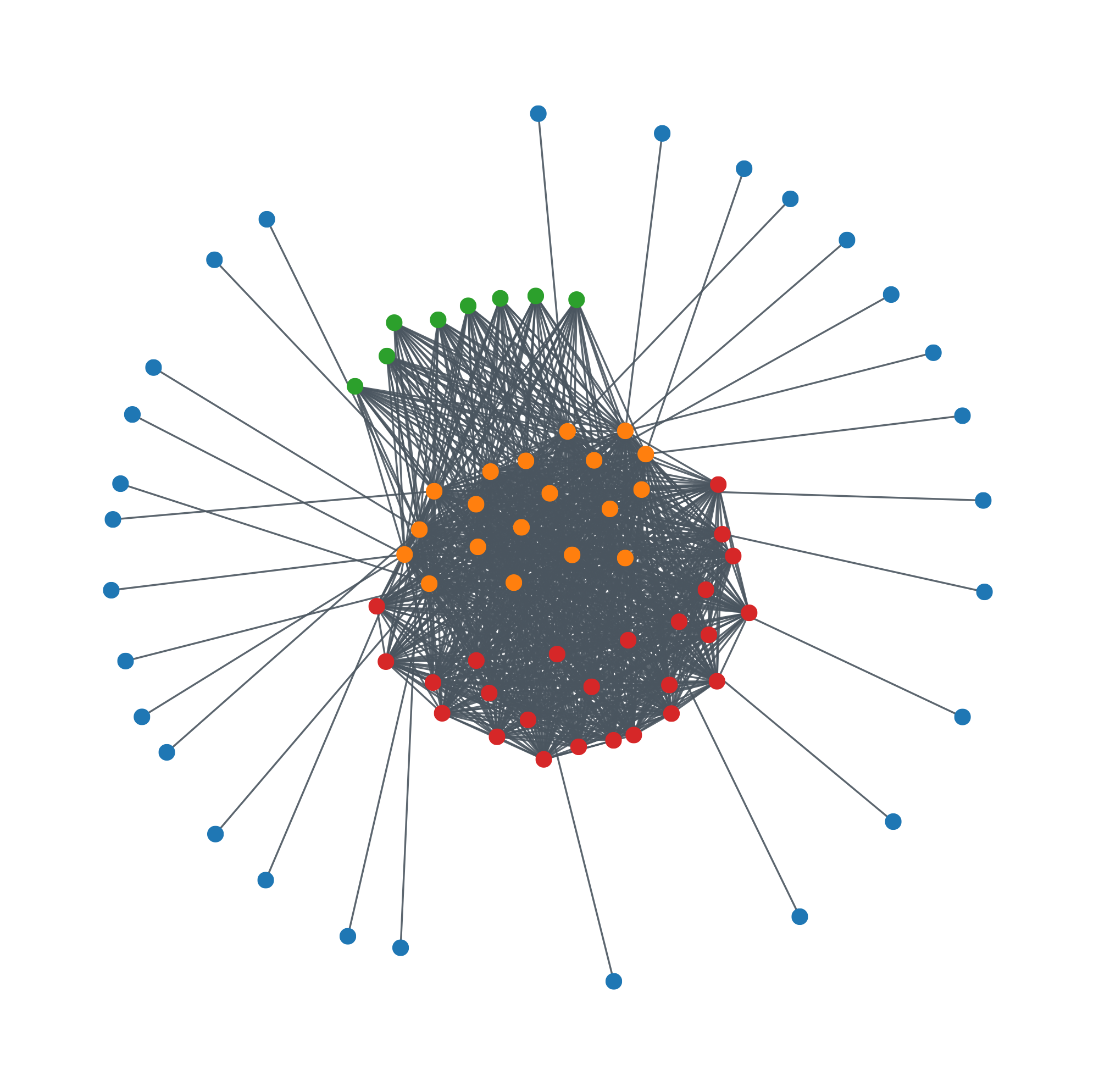}
    \end{subfigure}
    \caption{Visualizations of synthetic graphs generated by our method.}
    \label{fig:visualization}
\end{figure*}

To answer \textbf{RQ4}, we compare graphs generated by our prior with those generated by GraphPFN, the second-best method, using four statistics: (1) the degree assortativity coefficient~\cite{newman2002assortative} is the correlation between the degrees at the two ends of an edge. Positive values indicate that high-degree nodes tend to connect to each other, whereas negative values indicate that high-degree nodes tend to connect to low-degree nodes; (2) the global clustering coefficient~\cite{newman2001random} is the fraction of connected triples that close into triangles, so it measures triadic closure through how strongly nodes form local clusters; (3) the class-degree association is the fraction of degree variance explained by the node label;
(4) adjusted homophily compares the fraction of edges that connect nodes of the same class with the fraction expected from the class-size distribution.
For each method, we generate 200 graphs with the same number of nodes, sampled from $[100,2000]$, and keep the average degree close to 40, so that synthetic graphs from different methods are compared under matched size and density. 

Figure~\ref{fig:statistics} shows the distributions of the statistics. Across all four measures, our prior covers a broader range than GraphPFN, suggesting that it captures a wider variety of graph structures. Moreover, the degree assortativity coefficient and the global clustering coefficient of our graphs span a wide range, demonstrating that the prior can generate graphs with degree assortativity and triadic closure.

We further visualize graphs sampled from our synthetic prior in Figure~\ref{fig:visualization}, with node colors indicating their labels. The prior can generate several graph structures commonly observed in real-world data, including homophilic, heterophilic, and core-periphery graphs~\cite{borgatti2000models}. Our synthetic data generation algorithm can further capture a continuous core-periphery structure. For example, in the rightmost graph of the first row and the second graph from the right in the second row, connectivity gradually becomes sparser from the center to the periphery. Together, these results show that our prior can generate complex graphs exhibiting triadic closure, degree assortativity, and continuous core-periphery structure.

\section{Related Work}\label{sec:related}

\subsection{Graph Foundation Models}

GFMs aim to transfer knowledge across graph datasets and generalize to previously unseen graphs and tasks. From an architectural perspective, they can be broadly grouped into GNN-based~\cite{zhao2025graphany,finkelshtein2026equivariance}, LLM-based~\cite{chen2024exploring}, and hybrid GNN+LLM models~\cite{liuone}. 

\textbf{GNN-based models}. They exploit graph structure through topology-aware inductive biases. For example, GraphAny~\cite{zhao2025graphany} formulates inference on a new graph as an analytical LinearGNN solution and combines predictions from multiple such solutions using inductive attention. SAMGPT~\cite{yu2025samgpt} pre-trains a message-passing encoder on text-free graphs from multiple domains, inserting learnable structure tokens into each layer to align topology-dependent aggregation, and adapts the frozen encoder to node classification on an unseen graph with dual structural prompts.

\textbf{LLM-based models}. They incorporate semantic information through natural language processing. InstructGLM~\cite{ye2024language}, for instance, extends the vocabulary of a pre-trained language model by encoding node features as distinct tokens, facilitating cross-modal graph understanding. LLaGA~\cite{chen2024llaga} reorganizes the neighborhood of a target node into a structure-aware sequence and maps it into the token embedding space of a frozen language model, which then performs node classification on citation and product graphs.

\textbf{Hybrid GNN+LLM models}. They integrate structural and textual information. For example, OFA~\cite{liuone} aligns textual descriptions of nodes and edges into a shared embedding space through language modeling. GraphGPT~\cite{tang2024graphgpt} pairs this with the opposite interface: a graph encoder produces structural embeddings, a lightweight projector turns them into graph tokens, and two-stage instruction tuning teaches the language model to classify nodes in both supervised and zero-shot settings.

Despite these advances, existing approaches remain constrained by their underlying architectures, including limited pre-training scale~\cite{liu2025graph} and reliance on textual alignment~\cite{eremeev2026turning,eremeev2026graphpfn}.

\subsection{Prior-data Fitted Networks}

PFNs offer an alternative to conventional pre-training by learning in-context prediction from synthetic tasks sampled from prior distributions over data-generating processes. TabPFN~\cite{hollmanntabpfn,hollmann2025accurate} and LimiX~\cite{zhang2025limix,zhang2026limix} have demonstrated strong performance on tabular tasks, while other works, such as TuneTables~\cite{feuer2024tunetables} and BETA~\cite{liutabpfn}, have extended PFN approaches to larger and real-world datasets. Recent work has brought this paradigm to graph tasks~\cite{eremeev2026turning,liu2026tabularfoundationmodelsstrong,eremeev2026graphpfn} along two broad directions. The first direction adapts tabular PFNs to graph prediction. For example, G2T-FM~\cite{eremeev2026turning} and TAG~\cite{hayler2025bringing} apply tabular PFN approaches to graph tasks. However, these methods rely heavily on hand-crafted features and lack large-scale pre-training on diverse graph data, which limits their ability to capture complex graph patterns~\cite{eremeev2026graphpfn}. The second direction develops graph-specific priors for multi-domain node classification, as in GraphPFN~\cite{eremeev2026graphpfn} and NodePFN~\cite{choi2026nodepfn}. Despite these advances, existing methods still encounter severe difficulties in jointly modeling causal relationships within and across node and edge sets, generating diverse graph structures, and designing suitable architectures that can learn from complex synthetic datasets. 

\section{Conclusion}\label{sec:con}
In this work, we propose \name, a graph PFN pre-trained on synthetic data generated from an SCM-based prior that jointly models nodes and graph topology. Our method effectively achieves this goal by integrating three key components, including the \scm, the \scmdag~module, and the \architecture. Extensive experiments on 23 real-world benchmarks demonstrate that \name~outperforms both conventional GFMs and PFN-based GFMs, without utilizing real-world graph data for pre-training.

\section*{Ethics and Privacy Statement}

This study uses benchmark graph datasets and synthetically generated data and does not involve human participants or the collection of new personal information. We have not identified any specific ethical or privacy concerns arising from this work.

\bibliographystyle{ACM-Reference-Format}
\bibliography{base}

\appendix

\section{Theory}

\subsection{Graph Generation Strategy Expressiveness}
\label{subsec:edge-independence}

Let $n\ge 3$, and index the nodes by $u$ and $v$.
At each layer $\ell$, $V^{\ell}_{A}$ is the variable representing the edges, and $(V^{\ell}_{A})_{uv}$ is its coordinate on nodes $u$ and $v$.
Write $E^{(\ell)}_{uv}=(V^{\ell}_{A})_{uv}$ and $E^{(\ell)}=(E^{(\ell)}_{uv})_{1\le u<v\le n}$.
Let $E$ collect these coordinates over all layers.
The layered assignment updates the variables by Eq.~\eqref{eq:dag-layer}.
Let $b_{u}$ be the discrete block assignment of node $u$, and set $b=(b_{u})_{u=1}^{n}$.
For a pair $e=\{u,v\}$, write $b_{e}=(b_{u},b_{v})$.
As in the problem formulation, draw $p\sim\Pi$ and then draw a fully labeled graph from $p$.
The initialization of $V^{0}_{G}$ is drawn from an ER or a Bernoulli SBM prior, and the strategy realizes $p$ by applying Eq.~\eqref{eq:layer-maps} and then designating the observables.
Here Eq.~\eqref{eq:layer-maps} is the decomposition of the layer map into $g_{X}$, $g_{Y}$, and $g_{A}$.

\begin{definition}[Independently given discrete block assignments]
\label{def:block-indep}
$E$ is drawn independently given discrete block assignments if
\begin{equation}
\label{eq:block-factor}
P(E\mid b)
=
\prod_{1\le u<v\le n} P(E_{uv}\mid b_{u},b_{v})
\end{equation}
almost surely.
$E$ is drawn independently when Eq.~\eqref{eq:block-factor} holds for an almost surely constant $b$.
A constant assignment is the ER model.
An assignment for which each edge probability depends on the endpoints only through $(b_{u},b_{v})$ is the Bernoulli SBM.
A distribution $p$ lies in $\mathcal{P}_{\mathrm{ER}}$ or $\mathcal{P}_{\mathrm{SBM}}$ when the edges of a graph drawn from $p$ are generated by the corresponding model.
$\mathcal{P}_\mathrm{SCM}$ is the set of distributions $p$ realizable by the SCM-based strategy.
\end{definition}

\begin{proposition}
\label{prop:block-consequences}
Assume Eq.~\eqref{eq:block-factor}.
For distinct pairs $e,e^{\prime}$,
\begin{equation}
\label{eq:zero-cov}
\operatorname{Cov}\bigl(E_{e},E_{e^{\prime}}\mid b\bigr)
= 0
\end{equation}
almost surely.
The same identity holds for bounded measurable functions of these two coordinates.
Thus, for any finite set $\mathcal{T}$ of pairs, and for a measurable set $\mathcal{A}_{e}$ associated with each pair $e\in\mathcal{T}$,
\begin{equation}
\label{eq:motif-factor}
P(E_{e}\in \mathcal{A}_{e}\ \text{for all}\ e\in \mathcal{T}\mid b)
=
\prod_{e\in \mathcal{T}} P(E_{e}\in \mathcal{A}_{e}\mid b).
\end{equation}
Each coordinate is then determined by the block assignments of its own endpoints and by a noise variable private to that coordinate.
Given $b$, these private noises are independent.
\end{proposition}

\begin{proof}
Eq.~\eqref{eq:block-factor} is a product regular conditional distribution, so Eq.~\eqref{eq:zero-cov} and Eq.~\eqref{eq:motif-factor} hold.
Each factor $P(E_{uv}\mid b_{u},b_{v})$ is drawn on its own, which is the private-noise form.
\end{proof}

\begin{assumption}[Block assignments do not read edges]
\label{ass:A}
Each $b_{u}$ is generated from node-level initial variables.
No coordinate of $b$ is a descendant of any coordinate of $E$.
\end{assumption}

\begin{assumption}[Cross-Edge Path]
\label{ass:B}
There exist a layer $\ell$ and distinct nodes $u,v,w$ such that Eq.~\eqref{eq:dag-layer} contains the path
\begin{equation}
\label{eq:cross-path}
E^{(\ell)}_{uv}
\longrightarrow
(V^{\ell+1}_{X})_{u}
\longrightarrow
E^{(\ell+2)}_{uw}.
\end{equation}
The middle term is the coordinate of $V^{\ell+1}_{X}$ on node $u$.
The first arrow is $g_{X}(\cdot)$, and the second is $g_{A}(\cdot)$ at the next layer.
No coordinate of $b$ lies on this path.
\end{assumption}

\begin{assumption}[Non-Degenerate Transmission]
\label{ass:C}
On a set of mechanism parameters with positive probability, the map in Eq.~\eqref{eq:cross-path} is non-constant:
with $b$ and every exogenous variable that does not feed $E^{(\ell)}_{uv}$ held fixed, $E^{(\ell+2)}_{uw}$ still changes with $E^{(\ell)}_{uv}$.
When the two maps are differentiable, this is
\begin{equation}
\label{eq:cross-derivative}
\frac{\partial E^{(\ell+2)}_{uw}}{\partial E^{(\ell)}_{uv}}
\neq 0
\end{equation}
on a non-null set.
\end{assumption}

\begin{lemma}
\label{lem:outside-sbm}
Under Assumptions~\ref{ass:A}--\ref{ass:C},
\begin{equation}
\label{eq:outside}
P(E\mid b)
\neq
\prod_{1\le u<v\le n} P(E_{uv}\mid b_{u},b_{v})
\end{equation}
on a set of positive probability.
Hence, on that set, the edge-related variables are not drawn independently, and they are not drawn independently given discrete block assignments.
Their conditional distribution lies outside the ER and Bernoulli SBM priors.
\end{lemma}

\begin{proof}
Suppose Eq.~\eqref{eq:block-factor} holds almost surely.
For the pairs in Eq.~\eqref{eq:cross-path}, Proposition~\ref{prop:block-consequences} and Assumption~\ref{ass:A} give
\begin{equation*}
P\bigl(E^{(\ell+2)}_{uw}\bigm| b,E^{(\ell)}_{uv}\bigr)
=
P\bigl(E^{(\ell+2)}_{uw}\bigm| b\bigr)
\quad\text{a.s.}
\end{equation*}
Assumption~\ref{ass:C} gives a positive-probability set of parameters on which Eq.~\eqref{eq:cross-derivative} holds, with $b$ and every exogenous input off $E^{(\ell)}_{uv}$ fixed. On that set
\begin{equation*}
\frac{\partial E^{(\ell+2)}_{uw}}{\partial E^{(\ell)}_{uv}}\neq 0
\implies
P\bigl(E^{(\ell+2)}_{uw}\bigm| b,E^{(\ell)}_{uv}\bigr)
\neq
P\bigl(E^{(\ell+2)}_{uw}\bigm| b\bigr).
\end{equation*}
Hence Eq.~\eqref{eq:outside}.
The product in Eq.~\eqref{eq:block-factor} lies in $\mathcal{P}_{\mathrm{ER}}$ when $b$ is constant, and in $\mathcal{P}_{\mathrm{SBM}}$ when every factor is $P(E_{uv}\mid b_{u},b_{v})$.
\end{proof}

\begin{assumption}[The readout preserves the cross-edge effect]
\label{ass:D}
The observed adjacency is a measurable function of $E$.
On the positive-probability set of Assumption~\ref{ass:C}, the observed edges on the pairs in Eq.~\eqref{eq:cross-path} satisfy
\begin{equation}
\label{eq:obs-cov}
\operatorname{Cov}\bigl(E_{uv},E_{uw}\mid b\bigr)\neq 0.
\end{equation}
\end{assumption}

\begin{lemma}
\label{lem:observed-outside}
Under Assumptions~\ref{ass:A}--\ref{ass:D}, the edges of the observed graph are not drawn independently, nor independently given discrete block assignments.
The induced distribution $p$ lies in neither $\mathcal{P}_{\mathrm{ER}}$ nor $\mathcal{P}_{\mathrm{SBM}}$.
\end{lemma}

\begin{proof}
Lemma~\ref{lem:outside-sbm} and Assumption~\ref{ass:D} give a positive-probability set on which Eq.~\eqref{eq:obs-cov} holds.
Proposition~\ref{prop:block-consequences} gives that Eq.~\eqref{eq:block-factor} implies $\operatorname{Cov}\bigl(E_{uv},E_{uw}\mid b\bigr)=0$ almost surely.
Therefore Eq.~\eqref{eq:block-factor} fails on that set, and
\begin{equation*}
p\notin\mathcal{P}_{\mathrm{ER}}\cup\mathcal{P}_{\mathrm{SBM}}.
\end{equation*}
\end{proof}

\begin{lemma}
\label{lem:identity}
Suppose that in Eq.~\eqref{eq:layer-maps} each map is the identity on its own variable and does not depend on the other two inputs,
\begin{equation*}
g_{X}(V^{\ell}_{X},V^{\ell}_{Y},V^{\ell}_{A})=V^{\ell}_{X},
\quad
g_{Y}(V^{\ell}_{X},V^{\ell}_{Y},V^{\ell}_{A})=V^{\ell}_{Y},
\quad
g_{A}(V^{\ell}_{X},V^{\ell}_{Y},V^{\ell}_{A})=V^{\ell}_{A}.
\end{equation*}
Then $V^{\ell}_{G}=Z^{0}_{G}$ at every layer, so $p$ equals the distribution of the initialization.
An ER initialization gives $p\in\mathcal{P}_{\mathrm{ER}}$, and a Bernoulli SBM initialization gives $p\in\mathcal{P}_{\mathrm{SBM}}$.
\end{lemma}

\begin{proof}
Eq.~\eqref{eq:layer-maps} and these identities give $V^{\ell+1}_{G}=V^{\ell}_{G}$.
With $V^{0}_{G}=Z^{0}_{G}$,
\begin{equation*}
V^{\ell}_{G}=Z^{0}_{G}.
\end{equation*}
The observables are coordinates of $Z^{0}_{G}$, so $p$ is the distribution of $Z^{0}_{G}$.
An initialization from an ER prior gives $p\in\mathcal{P}_{\mathrm{ER}}$, and an initialization from a Bernoulli SBM prior gives $p\in\mathcal{P}_{\mathrm{SBM}}$.
\end{proof}

\begingroup
\renewcommand{\thetheorem}{\ref{thm:graphs}}
\begin{theorem}[Graph Generation Strategy Expressiveness]
Let $\mathcal{P}_\mathrm{SCM}$ be the set of distributions $p$ of graphs realizable by our SCM graph generation strategy, and let $\mathcal{P}_{\mathrm{ER}}$ and $\mathcal{P}_{\mathrm{SBM}}$ be the sets realizable by ER and Bernoulli SBM priors, respectively. Under Assumptions~\ref{ass:A}--\ref{ass:D}, the SCM-based graph generation strategy is more expressive than the ER and Bernoulli SBM priors:
\begin{equation*}
\mathcal{P}_{\mathrm{ER}}\cup\mathcal{P}_{\mathrm{SBM}}\subsetneq\mathcal{P}_\mathrm{SCM}.
\end{equation*}
\end{theorem}
\endgroup
\addtocounter{theorem}{-1}

\begin{proof}
Lemma~\ref{lem:identity} gives
\begin{equation*}
\mathcal{P}_{\mathrm{ER}}\cup\mathcal{P}_{\mathrm{SBM}}\subseteq\mathcal{P}_\mathrm{SCM}.
\end{equation*}
Lemma~\ref{lem:observed-outside} gives a positive-probability set of mechanisms with
\begin{equation*}
p\in\mathcal{P}_\mathrm{SCM}\setminus\bigl(\mathcal{P}_{\mathrm{ER}}\cup\mathcal{P}_{\mathrm{SBM}}\bigr).
\end{equation*}
Therefore, $\mathcal{P}_{\mathrm{ER}}\cup\mathcal{P}_{\mathrm{SBM}}\subsetneq\mathcal{P}_\mathrm{SCM}$.
\end{proof}

For the maps in Eq.~\eqref{eq:dag-layer},
\begin{equation*}
E^{(\ell)}_{uv}
\xrightarrow{g_{X}}
(V^{\ell+1}_{X})_{u}
\xrightarrow{g_{A}}
E^{(\ell+2)}_{uw},
\end{equation*}
which is Eq.~\eqref{eq:cross-path}, so Assumption~\ref{ass:B} holds.
Since $b$ is node-level initial and no coordinate of $b$ is a descendant of $E$, Assumption~\ref{ass:A} holds.
On a positive-probability set of weights of $g_{X}$ and $g_{A}$,
\begin{equation*}
\frac{\partial E^{(\ell+2)}_{uw}}{\partial E^{(\ell)}_{uv}}\neq 0,
\end{equation*}
so Assumption~\ref{ass:C} holds and Lemma~\ref{lem:outside-sbm} applies.
The observed adjacency keeps a fixed number of edges, selected by the magnitudes of $E$. On that set Eq.~\eqref{eq:obs-cov} holds, so Assumption~\ref{ass:D} and Lemma~\ref{lem:observed-outside} apply.
Under the identities in Lemma~\ref{lem:identity}, $p\in\mathcal{P}_{\mathrm{ER}}\cup\mathcal{P}_{\mathrm{SBM}}$.
Theorem~\ref{thm:graphs} follows.

\subsection{Feature Permutation Invariance}
\label{subsec:feature-permutation}

As in Section~\ref{sec:problem}, $\mathcal{G}=(\mathcal{V},\mathbf{A},\mathbf{X},\mathbf{y}_{\mathrm{train}})$, where $\mathbf{X}\in\mathbb{R}^{|\mathcal{V}|\times d}$.
Nodes in $\mathcal{V}_{\mathrm{train}}$ are the context, and nodes in $\mathcal{V}_{\mathrm{test}}$ are the queries.
With the context placed first, these nodes are $1{:}k$ and $(k+1){:}n$, where $n=|\mathcal{V}|$ and $k=|\mathcal{V}_{\mathrm{train}}|$.
Dimension $d+1$ of $\mathbf{H}\in\mathbb{R}^{n\times(d+1)\times d_h}$ is the label dimension.
Let $\pi$ be a permutation of the feature indices $\{1,\ldots,d\}$.
Denote $\mathcal{G}_{\pi}=(\mathcal{V},\mathbf{A},\pi(\mathbf{X}),\mathbf{y}_{\mathrm{train}})$.
For $j\le d$, feature dimensions move by $\pi^{-1}$ and the label dimension stays fixed:
\begin{equation}
\label{eq:pi-action}
\begin{aligned}
(\pi(\mathbf{X}))_{i,j}&=\mathbf{X}_{i,\pi^{-1}(j)},\\
(\pi(\mathbf{H}))_{i,j}&=\mathbf{H}_{i,\pi^{-1}(j)},\\
(\pi(\mathbf{H}))_{i,d+1}&=\mathbf{H}_{i,d+1}.
\end{aligned}
\end{equation}
One shared feature encoder is applied to each $\mathbf{X}_{i,j}$ alone.
The message-passing layer $\mathrm{MP}$ in Eq.~\eqref{eq:mp} is shared across features and does not read or write dimension $d+1$.
Attention modules do not share parameters.
For the feature-attention module in Eq.~\eqref{eq:feature-attn}, let $a_{j,t}(\mathbf{H}_{i})$ be the contribution of key dimension $t$ to query dimension $j$ at node $i$.
It is determined by $\mathbf{H}_{i,j}$, $\mathbf{H}_{i,t}$, and the parameters of this module.
\begin{equation}
\label{eq:feature-attn-paper}
\operatorname{FeatureAttn}(\mathbf{H}_{i})_{j}
=
\sum_{t=1}^{d+1}a_{j,t}(\mathbf{H}_{i}).
\end{equation}

\begin{lemma}[Feature Attention]
\label{lem:feature-attn}
$\operatorname{FeatureAttn}(\pi(\mathbf{H}_{i}))=\pi(\operatorname{FeatureAttn}(\mathbf{H}_{i}))$.
\end{lemma}

\begin{proof}
Eq.~\eqref{eq:pi-action} permutes the feature dimensions and leaves dimension $d+1$ fixed.
Since $a_{j,t}$ does not use the dimension index, for $j\le d$ and $t\le d$,
$a_{j,t}(\pi(\mathbf{H}_{i}))=a_{\pi^{-1}(j),\pi^{-1}(t)}(\mathbf{H}_{i})$,
while $a_{j,d+1}(\pi(\mathbf{H}_{i}))=a_{\pi^{-1}(j),d+1}(\mathbf{H}_{i})$.
Substitute $t'=\pi^{-1}(t)$:
\begin{equation}
\label{eq:attn-reindex}
\begin{aligned}
\operatorname{FeatureAttn}(\pi(\mathbf{H}_{i}))_{j}
&=
\sum_{t=1}^{d}a_{\pi^{-1}(j),\pi^{-1}(t)}(\mathbf{H}_{i}) \\
&+
a_{\pi^{-1}(j),d+1}(\mathbf{H}_{i}) \\
&=
\operatorname{FeatureAttn}(\mathbf{H}_{i})_{\pi^{-1}(j)}.
\end{aligned}
\end{equation}
The label query is the same sum with $\pi^{-1}(j)$ replaced by $d+1$.
\end{proof}

\begin{lemma}[Remaining Block Maps]
\label{lem:rest}
LayerNorm, the following MLP, and a residual sum act dimension-wise with parameters independent of $j$, so each sends feature dimension $j$ of $\pi(\mathbf{H})$ to dimension $\pi^{-1}(j)$ and leaves dimension $d+1$ fixed.
Sample attention is Eq.~\eqref{eq:sample-attn}, with weights independent of $j$ on the feature dimensions.
For $j\le d$, message passing satisfies Eq.~\eqref{eq:mp-equiv}.
\begin{equation}
\label{eq:mp-equiv}
[\mathrm{MP}(\pi(\mathbf{H}),\mathbf{A})]_{:,j}
=
[\mathrm{MP}(\mathbf{H},\mathbf{A})]_{:,\pi^{-1}(j)}.
\end{equation}
\end{lemma}

\begin{proof}
A dimension-wise map reads only $\mathbf{H}_{i,j}$, so Eq.~\eqref{eq:pi-action} permutes its outputs.
Eq.~\eqref{eq:sample-attn} depends on dimension $j$ only through its rows, and the split after node $k$ is fixed.
Eq.~\eqref{eq:mp-equiv} is Eq.~\eqref{eq:mp} applied to dimension $\pi^{-1}(j)$.
Dimension $d+1$ is fixed by Eq.~\eqref{eq:pi-action}, so that dimension is unchanged.
\end{proof}

\begin{lemma}[Feature Cross-Reduction]
\label{lem:reduction}
Let $\mathbf{H}^{\mathrm{sa}}\in\mathbb{R}^{n\times(d+1)\times d_h}$ and $\mathbf{H}^{\mathrm{mp}}\in\mathbb{R}^{n\times d\times d_h}$ be the tensors in Eq.~\eqref{eq:sa} and Eq.~\eqref{eq:mp}.
For each node $i$, $\mathbf{H}^{\mathrm{gt}}_{i}$ is the feature-wise concatenation of $\mathbf{H}^{\mathrm{sa}}_{i}$ and $\mathbf{H}^{\mathrm{mp}}_{i}$: the first $d+1$ dimensions are $\mathbf{H}^{\mathrm{sa}}_{i}$, and the last $d$ dimensions are $\mathbf{H}^{\mathrm{mp}}_{i}$.
The same order defines $\pi(\mathbf{H}^{\mathrm{gt}}_{i})$ from $\pi(\mathbf{H}^{\mathrm{sa}}_{i})$ and $\pi(\mathbf{H}^{\mathrm{mp}}_{i})$, where $(\pi(\mathbf{H}^{\mathrm{mp}}))_{i,r}=\mathbf{H}^{\mathrm{mp}}_{i,\pi^{-1}(r)}$ for $r\le d$.
Let $a^{\mathrm{gt}}_{j,t}$ be the contribution of key dimension $t$ of $\mathbf{H}^{\mathrm{gt}}$ to query dimension $j$, using the parameters of feature cross-reduction in Eq.~\eqref{eq:reduction}.
It is determined by those two dimensions.
\begin{equation}
\label{eq:reduction-attn}
\operatorname{Attn}(\mathbf{H}^{\mathrm{sa}},\mathbf{H}^{\mathrm{gt}},\mathbf{H}^{\mathrm{gt}})_{i,j}
=
\sum_{t=1}^{2d+1}a^{\mathrm{gt}}_{j,t}.
\end{equation}
For $j\le d$,
\begin{equation}
\label{eq:reduction-equiv}
\operatorname{Attn}(\pi(\mathbf{H}^{\mathrm{sa}}),\pi(\mathbf{H}^{\mathrm{gt}}),\pi(\mathbf{H}^{\mathrm{gt}}))_{i,j}
=
\operatorname{Attn}(\mathbf{H}^{\mathrm{sa}},\mathbf{H}^{\mathrm{gt}},\mathbf{H}^{\mathrm{gt}})_{i,\pi^{-1}(j)}.
\end{equation}
\end{lemma}

\begin{proof}
In the first block, Eq.~\eqref{eq:pi-action} permutes the $d$ feature keys and leaves the label key fixed.
The remaining $d$ keys are permuted among themselves.
For $j\le d$, substitute $t'=\pi^{-1}(t)$ and $r'=\pi^{-1}(r)$:
\begin{equation}
\label{eq:reduction-reindex}
\begin{aligned}
&\operatorname{Attn}(\pi(\mathbf{H}^{\mathrm{sa}}),\pi(\mathbf{H}^{\mathrm{gt}}),\pi(\mathbf{H}^{\mathrm{gt}}))_{i,j}\\
=&
\sum_{t=1}^{d}a^{\mathrm{gt}}_{\pi^{-1}(j),\pi^{-1}(t)}
+
a^{\mathrm{gt}}_{\pi^{-1}(j),d+1}
+
\sum_{r=1}^{d}a^{\mathrm{gt}}_{\pi^{-1}(j),\,d+1+\pi^{-1}(r)}\\
=&
\operatorname{Attn}(\mathbf{H}^{\mathrm{sa}},\mathbf{H}^{\mathrm{gt}},\mathbf{H}^{\mathrm{gt}})_{i,\pi^{-1}(j)}.
\end{aligned}
\end{equation}
The label query is the same sum with $\pi^{-1}(j)$ replaced by $d+1$.
\end{proof}

\featureinvariance*

\begin{proof}
Each node is encoded as $d$ feature embeddings and one label embedding.
For $j\le d$, dimension $j$ is a shared function of $\mathbf{X}_{i,j}$ alone.
For $1\le i\le k$, dimension $d+1$ is a function of $(\mathbf{y}_{\mathrm{train}})_{i}$.
For a query node $i\ge k+1$, dimension $d+1$ does not depend on $\mathbf{X}$.
Thus $\mathbf{H}(\mathcal{G}_{\pi})=\pi(\mathbf{H}(\mathcal{G}))$.
Each block is feature attention, sample attention, message passing, and feature cross-reduction, with pre-normalized residual MLPs.
Lemmas~\ref{lem:feature-attn}--\ref{lem:reduction} preserve $\pi(\mathbf{H})$, so $\mathbf{H}_{i,d+1}(\mathcal{G}_{\pi})=\mathbf{H}_{i,d+1}(\mathcal{G})$.
For $i\ge k+1$ and a class $c$, the conditional probability of $(\mathbf{y}_{\mathrm{test}})_{i}=c$ under $f_{\theta}(\cdot\mid\mathcal{G})$ depends only on $\mathbf{H}_{i,d+1}$.
\end{proof}

\section{Pseudocode}\label{sec:pseudocode}

We provide the pseudocode of the \scmdag~module in Algorithm~\ref{alg:prior-guided-node-edge}.

\begin{algorithm}[t]
\caption{Prior-Guided Node-Edge Dependency Modeling}
\label{alg:prior-guided-node-edge}
\begin{algorithmic}[1]
\Require Number of prior-guided message-passing layers $L$
\Ensure Node features $X$, node labels $Y$, and adjacency $A$
\State Draw $Z^{0}_{G}=(Z^{0}_{X},Z^{0}_{Y},Z^{0}_{A})$ from a prior combining the ER model and an SBM, and set $V^{0}_{G}\gets Z^{0}_{G}=(V^{0}_{X},V^{0}_{Y},V^{0}_{A})$.
\For{$\ell=0,1,\ldots,L-1$}
    \State Update features by feature-to-feature and label-to-feature linear maps with standard or degree-based neighbor aggregation, $V^{\ell+1}_{X}\gets g_{X}(V^{\ell}_{X},V^{\ell}_{Y},V^{\ell}_{A})$; update labels by a feature-to-label linear map, a label residual, and label propagation, $V^{\ell+1}_{Y}\gets g_{Y}(V^{\ell}_{X},V^{\ell}_{Y},V^{\ell}_{A})$; update edges by a feature-similarity kernel, a label-similarity kernel or a Chung--Lu mechanism, and an adjacency residual, $V^{\ell+1}_{A}\gets g_{A}(V^{\ell}_{X},V^{\ell}_{Y},V^{\ell}_{A})$.
    \State Sample an activation $\sigma$, apply it to $V^{\ell+1}_{X}$, $V^{\ell+1}_{Y}$, and $V^{\ell+1}_{A}$, and inject Gaussian noise.
\EndFor
\State Collect every layer as $V_{X}=(V^{0}_{X},\ldots,V^{L}_{X})$, $V_{Y}=(V^{0}_{Y},\ldots,V^{L}_{Y})$, $V_{A}=(V^{0}_{A},\ldots,V^{L}_{A})$, and set $V_{G}=(V_{X},V_{Y},V_{A})$.
\State Designate a subset as observables and keep the rest latent: $X\subseteq V_{X}$, $Y\subseteq V_{Y}$, $A\subseteq V_{A}$.
\State \Return $X$, $Y$, $A$
\end{algorithmic}
\end{algorithm}

\section{Experimental Details}

\subsection{Datasets}\label{subsec:dataset}

\begin{table}
\centering
\caption{Statistics of the 23 test datasets. \#Edges counts each undirected edge once, with self-loops removed.}
\label{tab:dataset-statistics}
\begin{tabular}{l c c c c}
\toprule
Dataset & \#Nodes & \#Edges & \#Features & \#Classes\\
\midrule
Actor            & 7,600  & 26,659   & 932   & 5\\
AirBrazil        & 131    & 1,003    & 131   & 4\\
AirEU            & 399    & 5,993    & 399   & 4\\
AirUS            & 1,190  & 13,599   & 1,190 & 4\\
Amazon-Comp      & 13,752 & 245,861  & 767   & 10\\
Amazon-Photo     & 7,650  & 119,081  & 745   & 8\\
Amazon-Ratings   & 24,492 & 93,050   & 300   & 5\\
Chameleon        & 890    & 8,854    & 2,325 & 5\\
Citeseer         & 3,327  & 4,552    & 3,703 & 6\\
Coauthor CS      & 18,333 & 81,894   & 6,805 & 15\\
Coauthor Physics & 34,493 & 247,962  & 8,415 & 5\\
Cora             & 2,708  & 5,278    & 1,433 & 7\\
Cornell          & 183    & 277      & 1,703 & 5\\
DBLP             & 17,716 & 52,867   & 1,639 & 4\\
Deezer           & 28,281 & 92,752   & 128   & 2\\
Minesweeper      & 10,000 & 39,402   & 7     & 2\\
Pubmed           & 19,717 & 44,324   & 500   & 3\\
Questions        & 48,921 & 153,540  & 301   & 2\\
Squirrel         & 2,223  & 46,998   & 2,089 & 5\\
Texas            & 183    & 279      & 1,703 & 5\\
Tolokers         & 11,758 & 519,000  & 10    & 2\\
WikiCS           & 11,701 & 215,603  & 300   & 10\\
Wisconsin        & 251    & 450      & 1,703 & 5\\
\bottomrule
\end{tabular}
\end{table}

We evaluate our method on 23 datasets using the data splits of NodePFN~\cite{choi2026nodepfn}. For Chameleon and Squirrel, we use the filtered datasets from~\cite{platonov2023critical}. Dataset statistics are reported in Table~\ref{tab:dataset-statistics}.

\subsection{Implementation Details}\label{subsec:imple-detail}

\name~is a 9-layer transformer with embedding dimension 96. When node attributes are present, they are standardized over all nodes, clipped to $[-100,100]$, and reduced by a randomized SVD to 64 components. Each pre-training graph is drawn from the SCM prior. The node number $n$ is sampled uniformly from 100 to 2000, the number of features from 64 to 256, and the number of classes from 2 to 20. The graph is undirected, with the number of edges drawn uniformly from $n$ to $20n$. Nodes are randomly permuted, and the number of labeled training nodes is drawn uniformly from $10C$ to $\min(1000,\lfloor n/2\rfloor)$, where $C$ is the number of classes. The remaining nodes are the held-out targets. We minimize the mean cross-entropy of the 20-way predictive distribution on those held-out nodes.

\end{document}